\documentclass{new_tlp}
\usepackage{mathptmx}
\usepackage{amssymb}
\usepackage{amsmath}
\usepackage{url}
\usepackage[cmtip,arrow]{xy}
\usepackage[curve]{xypic}
\usepackage{graphicx}
\usepackage{xparse}
\usepackage{IEEEtrantools}
\usepackage{subfigure}
\usepackage[compact]{titlesec}

\newcommand{\down}[1]{\ensuremath{\downarrow\!{#1}}}
\newcommand{\cdotl}{\ensuremath{\!\cdot\!}}

\usepackage[usenames,dvipsnames]{xcolor}

\newcommand{\commentn}[1]{}

\newcommand{\change}[1]{#1}
\newcommand{\eqdef}{%
  \mathrel{\vbox{\offinterlineskip\ialign{%
    \hfil##\hfil\cr
    $\scriptscriptstyle\mathrm{def}$\cr
    \noalign{\kern1pt}
    $=$\cr
    \noalign{\kern-0.1pt}
}}}}

\newcommand{\fn}[1]{\hbox{\em #1}}
\newcommand{\LP}{{\bf LP}}

\def\qed{~\hfill$\Box$}
\newtheorem{definition}{Definition}
\newtheorem{theorem}{Theorem}
\newtheorem{example}{Example}
\newtheorem{lemma}{Lemma}[section]

\newtheorem{corollary}{Corollary}
\newtheorem{proposition}{Proposition}

\newenvironment{proofof}[1]{\vspace{5pt}\setlength{\parindent}{0cm}\setlength{\parskip}{0.2cm} {\em Proof of #1}.}{}
\renewenvironment{proof}{\setlength{\parindent}{0cm}\setlength{\parskip}{0.2cm} {\emph{Proof}}.}{\vspace{0.25cm}}
\def\qed{~\hfill{$\Box$}}

\newcommand{\tuple}[1]{\ensuremath{\langle #1 \rangle}}
\newcommand{\signature}{\tuple{At, Lb}}
\newcommand{\set}[1]{\ensuremath{\{#1\}}}
\newcommand{\setb}[1]{\ensuremath{\big\{\ #1\ \big\}}}

\newcommand{\setm}[2]{\ensuremath{\{\ #1\ \big|\ #2\ \}}}
\newcommand{\setbm}[2]{\ensuremath{\big\{\ #1\ \big|\ #2\ \big\}}}

\def\R{R}

\newcommand{\CS}[2]{\ensuremath{\mathbf{C}_{#1}^{#2}}}

\newcommand{\causes}{\CS{Lb}{}}

\newcommand{\VS}[2]{\ensuremath{\mathbf{V}_{#1}^{#2}}}
\newcommand{\VLb}{\VS{Lb}{}}

\newcommand{\In}{{\ensuremath{\mathbf{I}}}}

\newcommand{\Fi}{\ensuremath{\mathbf{F}}}

\newcommand{\sU}{\ensuremath{\mathcal{U}}}

\newcommand{\Id}{\ensuremath{\mathbf{I}}}

\newcommand{\cB}{\ensuremath{\underline{\mathfrak{B}}}}
\def\bottomI{\ensuremath{\mathbf{0}}}
\def\topI{\ensuremath{\mathbf{1}}}

\newcommand{\nop}[1]{}

\newcommand{\tPpp}[3]{\text{\mbox{\ensuremath{T_{#1}\!\uparrow^{\ #2}(\bottomI)(#3)}}}}
\newcommand{\tpp}[2]{\tPpp{P}{#1}{#2}}
\newcommand{\tPp}[2]{\text{\mbox{\ensuremath{T_{#1}\!\uparrow^{\ #2}(\bottomI)}}}}
\newcommand{\tp}[1]{\tPp{P}{#1}}

\def\Not{\hbox{\em not} \ }
\def\sneg{\hbox{\em } \ \neg}

\newcounter{programcount}
\newcommand{\newprog}{\refstepcounter{programcount}\ensuremath{P_{\arabic{programcount}}}}
\newcommand{\progref}[1]{\ensuremath{P_{\ref{#1}}}}

\DeclareDocumentCommand\impl{o m m}{\ensuremath{	\IfNoValueTF{#1}{}{{#1}:}{#3}\leftarrow{#2}}}

\newenvironment{examplecont}[1]{\vspace{5pt} \noindent \emph{Example~\ref{#1} (continued)} \\ \noindent}{\qed \vspace{5pt}}
\renewenvironment{examplecont}[1]{\begin{example}[Ex.~\ref{#1} continued]}{\end{example}}

\title[Theory and Practice of Logic Programming]
      {Causal Graph Justifications of Logic Programs\thanks{This research was partially supported by Spanish MEC project TIN2009-14562-C05-04, Xunta program INCITE 2011 and Inditex-University of Corunna 2013 grants, as well as by the Austrian Science Fund (FWF) project P24090.}}

\author[P. Cabalar, J. Fandinno \& M. Fink]
         {Pedro Cabalar, Jorge Fandinno\\
          Department of Computer Science\\
		  University of Corunna, Spain\\
		  \email{\{cabalar, jorge.fandino\}@udc.es}
		  \and Michael Fink\\
          Vienna University of Technology,\\
          Institute for Information Systems\\
          Vienna, Austria\\
		  \email{fink@kr.tuwien.ac.at}
		 }

\jdate{March 2003}
\pubyear{2003}
\pagerange{\pageref{firstpage}--\pageref{lastpage}}
\doi{S1471068401001193}

\begin{document}

\label{firstpage}

\maketitle

\titlespacing*{\section}{0pt}{8pt plus 4pt minus 4pt}{4pt plus 2pt minus 2pt}

\begin{abstract}
In this work we propose a multi-valued extension of logic programs under the stable models semantics where each true atom in a model is associated with a set of justifications. These justifications are expressed in terms of \emph{causal graphs} formed by rule labels and edges that represent their application ordering. For positive programs, we show that the causal justifications obtained for a given atom have a direct correspondence to (relevant) syntactic proofs of that atom using the program rules involved in the graphs. The most interesting contribution is that this causal information is obtained in a purely semantic way, by algebraic operations (product, sum and application) on a lattice of causal values whose ordering relation expresses when a justification is stronger than another. Finally, for programs with negation, we define the concept of \emph{causal stable model} by introducing an analogous transformation to Gelfond and Lifschitz's program reduct. As a result, default negation behaves as ``absence of proof'' and no justification is derived from negative literals, something that turns out convenient for elaboration tolerance, as we explain with a running example.
\end{abstract}


\section{Introduction}

An important difference between classical models and most Logic Programming (LP) semantics is that, in the latter, true atoms must be founded or justified by a given derivation. Consequently, falsity is understood as absence of proof: for instance, a common informal way reading for default literal $\Not p$ is ``there is no way to derive $p$.'' Although this idea seems quite intuitive, it actually resorts to a concept, the \emph{ways to derive} $p$,  outside the scope of the standard LP semantics. In other words, LP semantics point out whether there exists some derivation for an atom, but do not provide the derivations themselves, if several alternatives exist. 

However, such information on justifications for atoms can be of great interest for Knowledge Representation (KR), and especially, for dealing with problems related to causality. In the area of diagnosis, for instance, when a discrepancy between expected and observed behaviour is found, it may be convenient to not only exhibit a set of malfunctioning components as explanation, but also the way (a causal graph) in which these breakdowns have eventually caused the discrepancies. Another potential application area is legal reasoning where determining a legal responsability usually involves finding out which agent or agents have eventually caused a given result, regardless the chain of effects involved in the process. An important challenge in causal reasoning is the capability of not only deriving facts of the form
\mbox{``$A$ has caused $B$,''}
but also being able to represent them and reason about them. As an example, take the assertion: 
\begin{itemize}
\item[]  ``If somebody causes an accident, (s)he is legally responsible for that.''
\end{itemize}

\noindent This law does not specify the possible ways in which a person may cause an accident. Depending on a representation of the domain, the chain of events from the agent's action(s) to the final effect may be simple (a direct effect) or involve a complex set of indirect effects and defaults like inertia. Regarding representation of the above law, for instance, one might think of an informal rule: 
\begin{gather*}
responsible(X,Y)
	\leftarrow action(A),\ person(X),\ accident(Y),\ 
	\text{``\ensuremath{do(A,X) \ \hbox{\tt caused} \ occurs(Y) }'' }.
\end{gather*}
\noindent If the pseudo-literal ``$do(A,X) \ \hbox{\tt caused} \ occurs(Y)$'' actually corresponds to an explicit representation of all the possible ways of causing an accident, however, one immediately runs into a problem of \emph{elaboration tolerance}~\cite{McC98} --- adding new rules that causally connect $do(A,X)$ to $occurs(Y)$ (in a direct or indirect way) would force us to build new rules for $responsible(X,Y)$. What is needed instead, and what we actually propose as an eventual aim and future extension of our work, is to introduce, indeed, some kind of new LP literal ``$A$ caused $B$,'' with \emph{an associated semantics} capable of revealing causes $A$ of a given true atom $B$.

While not straightforward, the rewarding perspective of such a semantic approach is an extension of Answer Set Programming (ASP)~\cite{BrewkaET11} with causal literals capable of representing different kinds of causal influences (sufficient cause, necessary cause, etc). In this paper, we tackle the above issue and, as a first step and basic underlying requirement, develop a suitable semantics capable of associating causal justifications with each true atom. To this end, we propose a multi-valued extension of logic programs under the stable model semantics~\cite{GL88} where each true atom in a model is associated with a set of justifications in the form of \emph{causal graphs}. To further illustrate our motivation, consider the following example.

\begin{example}[From~\citeNP{Cabalar11}]\label{ex:prison}
Some country has a law $l$ that asserts that driving drunk is punishable with imprisonment. On the other hand, a second law $m$ specifies that resisting arrest has the same effect. The execution $e$ of a sentence establishes that a punishment implies imprisonment. Suppose that some person drove drunk and resisted to be arrested.\qed
\end{example}
We can capture this scenario with the following logic program \newprog\label{prg:prison}:
\[
\begin{array}{rl@{\hspace{25pt}}rl@{\hspace{25pt}}rl}
l: & punish \leftarrow drive,\ drunk  & 
m: & punish \leftarrow resist        &
e: & prison \leftarrow punish        \\
d: & drive &
k: & drunk &
r: & resist
\end{array}
\]
\noindent The least model of this positive program makes atom $prison$ true, so we know that there exists a possible derivation for it. In particular, two alternative justifications can be made, corresponding to the graphs in Figure~\ref{fig:prisonA}: driving drunk and, independently, resisting to authority (vertices and edges respectively corresponds with rule labels and their dependences).

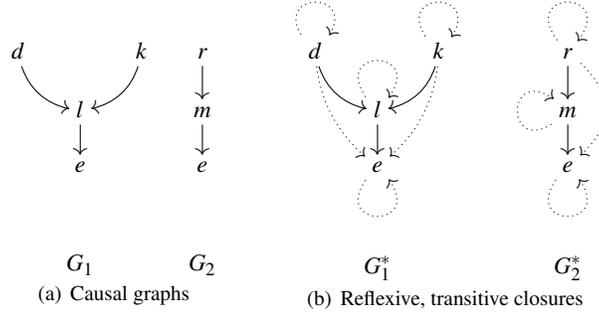
\begin{figure}\centering
\subfigure[Causal graphs]{ \label{fig:prisonA}
	$
	\xymatrix @-5mm {
	\\
	{d} \ar@/_/[dr] &             & {k} \ar@/^/[dl] & {r} \ar[d] \\
	                & {l} \ar[d]  &                 & {m} \ar[d] \\
	                & {e}         &                 & {e} \\ \\ 
	                & {G_1}       &                 & {G_2} 
	}
	$
}
\hspace*{20pt}
\subfigure[Reflexive, transitive closures]{ \label{fig:prisonB}
	$
	\xymatrix @-5mm {
	\\
	  {d} \ar@/_/[dr] \ar@{.>}@(ul,ur) \ar@/_/@{.>}[ddr] & & 
	  {k} \ar@/^/[dl] \ar@{.>}@(ul,ur) \ar@/^/@{.>}[ddl] &  \hspace{10pt} & 
	{r} \ar[d] \ar@{.>}@(ul,ur) \ar@/^{5mm}/@{.>}[dd] \\
	   & {l} \ar[d] \ar@{.>}@(ul,ur) & & &
	{m} \ar[d] \ar@{.>}@(dl,ul) \\
	   & {e} \ar@{.>}@(dl,dr) & & &
	{e} \ar@{.>}@(dl,dr)\\ \\
	            & {G^*_1} & & & {G^*_2}
	}
	$
}
\caption{Derivations $G_1$ and $G_2$ justifying atom $prison$ in program \progref{prg:prison}.}
\end{figure}
More specifically, we summarise our contributions as follows.
\begin{itemize}
\item We define a multi-valued semantics for (normal) logic programs based on causal graphs. An important result is that, despite of this semantic nature, we are able to show that causal values have a direct correspondence to (relevant) syntactic proofs using the program rules involved in the graphs (cf.~Section~\ref{sec:programs}).
\item We also define an ordering relation that specifies when a cause is \emph{stronger} than another, and show how causal values form a lattice with three associated algebraic operations: a product `$*$' representing conjunction or joint causation; a sum `$+$' representing alternative causes; and a non-commutative product `$\cdot$' that stands for rule application.
We study beneficial properties of these operations that allow manipulating and reasoning with causal values in an analytical way (cf.~Sections~\ref{sec:graphs} and~\ref{sec:alternatives}).
\end{itemize}

\vspace{-3pt}

\noindent Fostered by its algebraic treatment of causal values, our work facilitates the incorporation of dedicated, more specific causal expressions representing causal influence of different kinds.


\vspace{-3pt}

\section{Causes as graphs}\label{sec:graphs}

\vspace{-2pt}

In this and subsequent Section~\ref{sec:alternatives}, we introduce the lattice of causal values in two different steps. In a first step, we focus on the idea of an ``individual'' cause and then we proceed to explain the concept of causal value that allows collecting different alternative causes.

We begin recalling several graph definitions and notation.
A (directed) \emph{graph} is a pair $\tuple{V,E}$ where $V$ is a set of vertices $V$ and $E$ is a set of edges $E \subseteq V \times V$. In the following definitions, let $G=\tuple{V,E}$ and $G'=\tuple{V',E'}$ be two graphs. We say that $G$ is a \emph{subgraph} of $G'$, written $G \subseteq G'$, when $V \subseteq V'$ and $E \subseteq E'$. We write $G \cup G'$ to stand for the graph $\tuple{V \cup V',E \cup E'}$. We represent the reflexive and transitive closure of $G$ as $G^*$. Finally, we introduce a \emph{concatenation} operation $G \odot G'$ on graphs corresponding to a graph with vertices $V \cup V'$ and edges $E \cup E' \cup \{(x,y) \ | \ x \in V, y \in V'\}$.
Notice that, $G \cup G' \subseteq G \odot G'$, that is, the concatenation extends the union of graphs by adding all possible arcs that go from some node in $G$ to some node in $G'$.

\begin{definition}[Causal graph]
Given some set $Lb$ of (rule) labels,
a \emph{causal graph} (\emph{c-graph}) $G$ is a reflexively and transitively closed directed graph, i.e., $G^*=G$, whose vertices are labels, i.e.
$V\subseteq Lb$. We denote by $\causes$ the set of all possible causal graphs over $Lb$.\qed
\end{definition}

\noindent
Intuitively, the vertices correspond to rules involved in a derivation of a given atom (or formula), and the edges point out a (partial) ordering of application of rules in the derivation. Figure~\ref{fig:prisonA} shows two causal graphs with labels from  \progref{prg:prison}. 
Transitivity is crucial for interpreting the subgraph relation $G \subseteq G'$ as a way to express that $G'$ is redundant with respect to $G$. For instance, a graph $G_3$ formed by the single edge $(r,e)$ \emph{is not a subgraph} of $G_2$ but is simpler (does not require using $m$). This fact is captured by $G^*_3 \subseteq G^*_2$. Reflexivity is convenient for simpler definitions. For instance, the causal graph formed by a single label $l$ also has a single edge $(l,l)$---we call this an \emph{atomic} causal graphs and represent it just by its label. For simplicity, we will usually omit transitive and reflexive arcs when depicting a causal graph. For instance, taking $G_1$ and $G_2$ in Figure~\ref{fig:prisonA} as causal graphs actually amounts to considering the graphs shown in Figure~\ref{fig:prisonB}, where previously omitted arcs are shown as dotted lines. We define next a natural ordering relation among them.

\begin{definition}[Sufficient]
A causal graph $G$ is \emph{sufficient} for another causal graph $G'$, written $G \leq G'$, when $G \supseteq G'$.\qed
\end{definition}

\noindent Saying that $G$ is sufficient for $G'$ intuitively means that $G$ contains enough information to yield the same effect than $G'$, but perhaps more than needed (this explains $G \supseteq G'$). For this reason, we sometimes read $G \leq G'$ as ``$G'$ is \emph{stronger} than $G$.''

Since graphs with the subgraph relation form a poset, the set of causal graphs also constitutes a poset $\tuple{\causes,\leq}$ with a top element corresponding to the empty c-graph $G_\emptyset=\tuple{\emptyset,\emptyset}$. This stands for a kind of ``absolute truth'' and is of interest for including rules or facts one does not want to label, that is, their effect will not be traced in the justifications.

Any causal graph can be built up from labels (atomic causal graphs) using two basic operations: the product $G * G' \ \eqdef \ (G \cup G')^*$ that stands for union of causal graphs or \emph{joint interaction}, and the concatenation $G \cdot G'\ \eqdef \ (G \odot G')^*$ that captures their \emph{sequential application}. The reason for applying a closure is that the result of $G \cup G'$ and $G \odot G'$ does not need to be closed under transitivity. We can extend the product to any (possibly empty and possibly infinite) set of causal graphs $S$ so that
$\prod S \ \eqdef \ \big( \ \bigcup_{G \in S} G \ \big)^*$.


\begin{examplecont}{ex:prison}
The cause for the body of $l$ in \progref{prg:prison} is the product of causes for $drive$ and $drunk$, that is $d * k$ formed with vertices $\{d,k\}$ and edges $\{(d,d),(k,k)\}$. As a result, the explanation of the rule head, $punish$, is formed by the concatenation of its body cause $d * k$ with its label, that is $(d * k) \cdot l$. In its turn, this becomes the cause for the body of $e$ and so, we get the explanation $(d*k)\cdot l \cdot e$ for atom $prison$ represented as $G_1$ in Figure~\ref{fig:prisonB}. Similarly, $G_2$ corresponds to $r \cdot m \cdot e$. \qed
\end{examplecont}

\noindent When writing these causal expressions, we assume that `$\cdot$' has higher priority than `$*$'. Furthermore, we will usually omit `$\cdot$'  when applied to consecutive labels, so that $r \cdot m \cdot e$ will be abbreviated as $rme$. It is easy to see that $G * G' = G' * G$ while, in general, $G \cdot G' \neq G' \cdot G$, that is, concatenation is not commutative. Another observation is that $G \cdot G' \leq G * G'$, that is, concatenation is sufficient for the product, but the opposite does not hold in general. Moreover, in our running example, we can check that $(d*k) \cdot l$ is equal to $(d \cdot l) * (k \cdot l)$. In fact, application distributes over products and, as a result, we can identify any causal graph with the product of all its edges.
To conclude this section, we note that the set of causal graphs $\causes$ ordered by $\leq$ forms a
lower semilattice $\tuple{\causes, *}$, where the product constitutes the infimum.

\section{Alternative causes}\label{sec:alternatives}

Having settled the case for individual causes, let us now proceed to represent situations in which several alternative and independent causes can be found for an atom $p$. The obvious possibility is just using a \emph{set of causal graphs} for that purpose.
However, we should additionally disregard causes for which a stronger alternative exists. For instance, as we saw before, cause $r m e$ is sufficient for explaining $punish$ and therefore, it is also an alternative way to prove
this atom, but redundant in the presence of the stronger cause $r m$. This suggests to choose sets of $\leq$-maximal causal graphs as `truth values' for our multi-valued semantics. In principle, this is the central idea, although $\leq$-maximal causal graphs incur some minor inconveniences in mathematical treatment. For instance, if we collect different alternative causes by just using the union of sets of maximal causal graphs, the elements in the result need not be maximal. Besides, the operations of product and concatenation are expected to extend to the sets adopted as causal values. To address these issues, a more solid representation is obtained resorting to \emph{ideals} of causal graphs.

Given any poset $\tuple{A,\leq}$, an \emph{ideal} $I$ is any set $I \subseteq A$ satisfying\footnote{We use terminology from~\cite{Stumme97}. In some texts this is known as \emph{semi-ideal} or \emph{down-set} to differentiate this definition from the stronger case in which ideals are applied on a (full) lattice rather than a semi-lattice.}: if $x \in I$ and $y \leq x$ then $y \in I$. A compact way of representing an ideal $I$ is by using its set of maximal elements $S$, since the rest of $I$ contains \emph{all elements below} them. The \emph{principal ideal} of an individual element $x \in A$ is denoted as
$\downarrow x \ \eqdef \ \{y \in A \mid y \leq x\} $. We extend this notion for any set of elements $S$ so that
\mbox{$\downarrow S \ \eqdef \ \bigcup \{\downarrow x \mid x \in S\} = \{y \in A \mid y \leq x, \text{for some } x \in S\}$}.
Thus, we will usually represent an ideal $I$ as $\downarrow S$ where $S$ are the maximal
\footnote{Note that, in the case of causal graphs, the existence of maximal elements for the $\leq$-relation amounts to the existence of minimal elements for the subgraph relation, and this holds since the latter is well-founded.}
elements in $I$. In fact, maximal elements constitute the relevant information provided by the ideal, while keeping all other elements is convenient for simplicity of algebraic treament (but we do not assign a particular meaning to them).

\vspace{-1pt}

\begin{definition}[Causal Value]
Given a set of labels $Lb$, a \emph{causal value} is any ideal for the poset $\tuple{\causes,\leq}$.
We denote by $\VLb$ the set of causal values.\qed
\end{definition}

\vspace{-1pt}



\noindent
Product, concatenation and the $\leq$-relation are easily extended to any pair $U$, $U'$ of causal values respectively as: \mbox{$U*U' \ \eqdef \ U\cap U'$} and $U\cdot U' \ \eqdef \ \down{\setm{ G \cdot G' }{ G \in U \text{ and } G' \in U'}}$ and $U \leq U'$ iff $U \subseteq U'$. We also define addition as: \mbox{$U+U' \ \eqdef \ U\cup U'$} allowing to collect alternative causes. \change{Using terminology and results from lattice theory in~\cite{Stumme97} we can prove the following.}

\vspace{-1pt}

\begin{theorem}\label{theorem:freelattice}
{$\VLb$ forms a free, completely distributive lattice with join $+$ and meet $*$ generated by the lower semilattice $\tuple{\causes,*}$ with the injective homomorphism (or embedding) $\downarrow:\causes\longrightarrow\VLb$ 
.\qed}
\end{theorem}

\vspace{-1pt}

\noindent
\change{Essentially, this theorem means that the mapping $\downarrow$ from c-graph $G$ to its principal ideal $\down{G}$ is preserved for their respective products, \mbox{$\down{(G_1 * G_2)}\, =\, \down{G_1}\ * \down{G_2}$}, and  ordering relations: \mbox{$G_1\leq G_2$} (among c-graphs) iff $\down{G_1} \leq\ \down{G_2}$ (among causal values).}

\vspace{-1pt}

\begin{examplecont}{ex:prison}
The interpretation for $punish$ has two alternative causes $(d*k)\cdot l$ and $r m$ that become the causal values $\down{(d*k)\cdot l}$ and
$\down{r m}$. The causal value for $punish$ is then formed by their addition:
\vspace{-5pt}
\[
\down{(d*k)\cdot l} \ \ \ + \ \ \down{rm} \ \ 
= \ \ \down{(d*k)\cdot l} \ \ \ \cup \ \ \down{rm} \ \ 
= \ \ \down{\set{ \ (d*k)\cdot l, \ rm \ } }
\vspace{-5pt}
\]
This ideal contains, among others, the cause $r m e$, although it is not maximal due to $r m$:
\vspace{-5pt}
\begin{IEEEeqnarray*}{c+x*}
\down{ \set{ \ (d*k)\cdot l, \ rm \ } } \ \ \ \cup \ \down{ rme } \ = \ \down{  \set{ \ (d*k)\cdot l, \ rm, \ rme } } \ = \ 
\down{ \set{ \ (d*k)\cdot l, \ rm \ } } & $\Box$
\end{IEEEeqnarray*}
\end{examplecont}

\noindent
\change{The term completely distributive lattice in Theorem~\ref{theorem:freelattice} means that meet (resp. join) operation is defined for any infinite subset of $\VLb$ and distributes over infinite joins (resp. meets).
There is also a bottom element, the empty ideal $\emptyset$ (standing for ``falsity'') that will be denoted as $0$, and a top element, the ideal formed by the empty causal graph $\down{ G_\emptyset}=\causes$ (standing for ``absolute truth'') that is denoted as $1$ from now on.}
To improve readability, we introduce the syntactic notion of \emph{causal terms}, that allow representing the possible causal values without explicitly resorting to graphs or ideals.

\vspace{-3pt}

\begin{definition}[Causal term]
A \emph{(causal) term}, $t$, over a set of labels $Lb$, is recursively defined as one of the following expressions $t ::= l \ | \ \prod S \ | \ \sum S \ | \ t_1 \cdot t_2$ where $l \in Lb$, $t_1, t_2$ are in their turn causal terms and $S$ is a (possibly empty and possible infinite) set of causal terms. When $S$ is finite and non-empty, $S=\{t_1, \dots, t_n\}$ we write $\prod S$ simply as $t_1 * \dots * t_n$ and $\sum S$ as $t_1 + \dots + t_n$.\qed
\end{definition}

\noindent We assume that `$*$' has higher priority than `$+$'.
The causal value associated to a causal term is naturally obtained by the recursive application of its operators until we reach the level of labels, so that each label $l$ in the term actually stands for the principal ideal $\down{l}$.
When $S=\emptyset$, the union in $S$ corresponds to the bottom causal value that is, $\sum \emptyset = 0$. Analogously, the intersection of elements in $S=\emptyset$ corresponds to top causal value, i.e., $\prod \emptyset = 1$.

From now on, we will use causal terms as compact representations of causal values.
Individual causes (i.e. causal graphs) correspond to terms without addition (note that this also excludes $0$, the empty sum). 
Several interesting algebraic properties can be proved for causal values. In particular, Theorem~\ref{theorem:freelattice} guarantees that they form a free completely distributive lattice with respect to `$*$' and `$+$' satisfying the standard properties
such as associativity, commutativity, idempotence, absorption or distributivity on both directions\footnote{\change{The term ``free lattice'' in Theorem~\ref{theorem:freelattice} means that \emph{any} equivalence with $*$ and $+$ can be derived from these properties.}}. Besides, as usual, $0$ (resp. $1$) is the annihilator for `$*$' (resp. `$+$') and the identity for `$+$' (resp. `$*$'). More significantly, the main properties for `$\cdot$' are shown in Figure~\ref{fig:appl}.

\vspace{-10pt}

\begin{figure}[htbp]
\begin{center}
\newcommand{\titleSep}{-5pt}
\newcommand{\contentSep}{-15pt}
\newcommand{\rowSep}{5pt}
$
\begin{array}{c}
\hbox{\em Associativity}\vspace{\titleSep}\\
\hline\vspace{\contentSep}\\
\begin{array}{r@{\ }c@{\ }r@{}c@{}l c r@{}c@{}l@{\ }c@{\ }l@{\ }}
t & \cdot & (u & \cdot & w) & = & (t & \cdot & u) & \cdot & w\\
\\
\end{array}
\end{array}
$
\ \ \ \
$
\begin{array}{c}
\hbox{\em Absorption}\vspace{\titleSep}\\
\hline\vspace{\contentSep}\\
\begin{array}{r@{\ }c@{\ }c@{\ }c@{\ }l c r@{\ }c@{\ }r@{\ }c@{\ }c@{\ }c@{\ }c@{\ }l@{\ }}
&& t &&& = & t & + & u & \cdot & t & \cdot & w \\
u & \cdot & t & \cdot & w & = & t & * & u & \cdot & t & \cdot & w
\end{array}
\end{array}
$
\ \ \ \
$
\begin{array}{c}
\hbox{\em Identity}\vspace{\titleSep}\\
\hline\vspace{\contentSep}\\
\begin{array}{rc r@{\ }c@{\ }l@{\ }}
t & = & 1 & \cdot & t\\
t & = & t & \cdot & 1
\end{array}
\end{array}
$
\ \ \ \
$
\begin{array}{c}
\hbox{\em Annihilator}\vspace{\titleSep}\\
\hline\vspace{\contentSep}\\
\begin{array}{rc r@{\ }c@{\ }l@{\ }}
0 & = & t & \cdot & 0\\
0 & = & 0 & \cdot & t\\
\end{array}
\end{array}
$
\\
\vspace{\rowSep}
$
\begin{array}{c}
\hbox{\em Indempotence}\vspace{\titleSep}\\
\hline\vspace{\contentSep}\\
\begin{array}{r@{\ }c@{\ }l@{\ }c@{\ }l }
l & \cdot & l  & = & l\\
\\
\\
\end{array}
\end{array}
$
\hspace{.05cm}
$
\begin{array}{c}
\hbox{\em Addition\ distributivity}\vspace{\titleSep}\\
\hline\vspace{\contentSep}\\
\begin{array}{r@{\ }c@{\ }r@{}c@{}l c r@{}c@{}l@{\ }c@{\ }r@{}c@{}l@{}}
t & \cdot & (u & + & w) & = & (t & \cdot & u) & + & (t & \cdot & w)\\
( t & + & u ) & \cdot & w & = & (t & \cdot & w) & + & (u & \cdot & w)\\ \\
\end{array}
\end{array}
$
\hspace{.05cm}
$
\begin{array}{c}
\hbox{\em Product\ distributivity}\vspace{\titleSep}\\
\hline\vspace{\contentSep}\\
\begin{array}{rcl}
c \cdot d \cdot e & = & (c \cdot d) * (d \cdot e) \ \ \ \hbox{with} \ d \neq 1 \\
c \cdot (d*e)     & = & (c \cdot d) * (c \cdot e) \\
(c*d) \cdot e     & = & (c \cdot e) * (d \cdot e)
\end{array}
\end{array}
$
\end{center}
\vspace{-5pt}
\caption{Properties of the `$\cdot$' operator ($c,d,e$ are terms without `$+$' and $l$ is a label).}
\label{fig:appl}
\end{figure}

\vspace{-15pt}
\section{Positive programs and minimal models} \label{sec:programs}

Let us now reconsider logic programs and provide a semantics based on the causal values we have just defined. For the syntax, we recall standard LP definitions, just slightly extending it by introducing rule labels. A \emph{signature} is a pair \signature\ of sets that respectively represent a set of \emph{atoms} (or \emph{propositions}) and a set of \emph{labels}. As usual, a \emph{literal} is defined as an atom $p$ (positive literal) or its default negation $\Not p$ (negative literal). In this paper, we will concentrate on programs without disjunction in the head (leaving its treatment for future work).

\begin{definition}[Causal logic program]\label{def:causal.P}
Given a signature $\langle At,Lb\rangle$, a \emph{(causal) logic program} $P$ is a (possible infinite) set of rules of the form:
\vspace{-12pt}
\begin{eqnarray}
t: H \leftarrow B_1, \dotsc, B_n, \label{f:rule} 
\end{eqnarray}

\vspace{-5pt}
\noindent where $t\in Lb\cup\set{1}$, $H$ is an atom (the \emph{head} of the rule) and $B_1, \dotsc, B_n$ are literals (the \emph{body}).\qed
\end{definition}

\noindent For any rule $\R$ of the form \eqref{f:rule} we define $label(\R)\eqdef t$. We denote by $head(\R) \eqdef H$ its \emph{head}, and by $body(\R)\eqdef \{B_1,\dots,B_n\}$ its \emph{body}. When $n=0$ we say that the rule is a \emph{fact} and omit the symbol `$\leftarrow$.' When $t \in Lb$ we say that the rule is labelled; otherwise $t=1$ and we omit both $t$ and `$:$'. By these conventions, for instance, an unlabelled fact $p$ is actually an abbreviation of $(1: p \leftarrow )$. A logic program $P$ is \emph{positive} if it contains no default negation. A program is \emph{uniquely labelled} if no pair of labelled rules share the same label, and \emph{completely labelled} if, additionally, all rules are labelled. For instance, \progref{prg:prison} is completely labelled.

Given a signature \signature\ a \emph{causal interpretation} is a mapping $I:At~\longrightarrow~\VLb$ assigning a causal value to each atom.
For any interpretations
$I,$ $J$,
we say that
$I\leq J$ when $I(p) \leq J(p)$ for each atom $p \in At$.
Hence, there is a
\mbox{$\leq$-bottom} (resp. $\leq$-top) interpretation \bottomI\
(resp. \topI) that stands for the interpretation mapping each atom $p$ to $0$ (resp. $1$). The value assigned to a negative literal $\Not p$ by an interpretation $I$, denoted as $I(\Not p)$, is  defined as: $I(\Not p) \eqdef 1$ if
\mbox{$I(p)=0$}; and $I(\Not p)\eqdef 0$ otherwise.
An interpretation is \emph{two-valued} if it maps all atoms to $\{0,1\}$.
Furthermore, for any causal interpretation, its corresponding two-valued interpretation, written $I^{cl}$, is defined so that for any atom $p$: $I^{cl}(p) \eqdef 0$ if $I(p)=0$; and $I^{cl}(p)\eqdef 1$ otherwise.

\begin{definition}[Causal model]\label{def:causal.M}
Given a positive causal logic program $P$, a causal interpretation $I$ is a \emph{causal model}, in symbols $I \models P$, if and only if, for each rule $\R \in P$ of the form \eqref{f:rule}, the following condition holds:
\begin{IEEEeqnarray*}{c+x*}
\big( \ I(B_1) * \dotsc * I(B_n) \ \big) \cdot t \leq I(H) & $\Box$
\end{IEEEeqnarray*}
\end{definition}

\begin{examplecont}{ex:prison}
Take rule $l$ from program \progref{prg:prison} and let $I$ be such that $I(drive)=d$ and $I(drunk)=k$. Then $I$ will be a model of $l$ when $(d*k)\cdot l \leq I(punish)$. In particular, this holds when \mbox{$I(punish)=(d*k)\cdot l+ r\cdot m$} which was the value we expected for that atom. But it would also hold when, for instance, $I(punish)=l+m$ or $I(punish)=1$. The inequality in Definition~\ref{def:causal.M} is important to accommodate possible additional facts such as $(l: punish)$ or even $(1:punish)$ in the program.\qed
\end{examplecont}

\noindent
The fact that any $I(punish)$ greater than $(d*k)\cdot l+r\cdot m$ also becomes a model clearly points out the need for selecting \emph{minimal} models. In fact, as it is the case for non-causal programs, positive programs have a $\leq$-least model that can be computed by iterating an extension of the well-known \emph{direct consequences operator}~\cite{vEK76}.

\begin{samepage}
\begin{definition}[Direct consequences]\label{def:tp}Given a positive logic program $P$ over signature $\tuple{At,Lb}$, the operator of \emph{direct consequences} is a function $T_P : \In  \longrightarrow \In$ such that, for any causal interpretation $I$ and any atom $p \in At$:
\begin{align*}
T_P(I)(p)
	\ \eqdef \ \sum \big\{ \ \big( \ I(B_1) * \dotsc * I(B_n) \ \big) \cdot t \ \mid \ (t: p \leftarrow B_1, \dotsc, B_n ) \in P \ \big\}
\end{align*}
\end{definition}
\end{samepage}

\begin{theorem}\label{theorem:tp.properties}
Let $P$ be a (possibly infinite) positive logic program with $n$ causal rules.
Then,
($i$) $\mathit{lfp}(T_P)$ is the least model of $P$, and
($ii$)  $\mathit{lfp}(T_P)=\tp{\omega}=\tp{n}$.\qed
\end{theorem}

\noindent The proof of this theorem relies on an encoding of causal logic programs into \emph{Generalized Annotated Logic Programming} (GAP)~\cite{KiferS92} and applying existing results for that general multi-valued LP framework. Theorem~\ref{theorem:tp.properties} just guarantees that the least fixpoint of $T_P$ is well-behaved, but does not explain the nature of the obtained causal values. We illustrate next that these values have a direct relation to the syntactic idea of \emph{proof} in a positive program.

\begin{definition}\label{def:proof}
Given a positive program $P$, a \emph{proof} $\pi(p)$ of an atom $p$ can be recursively defined as a derivation: 
\begin{eqnarray*}
\pi(p) & \eqdef & \frac{\pi(B_1) \ \dotsc \ \pi(B_n)}{p} \  (\R),
\end{eqnarray*}
\noindent where $\R \in P$ is a rule with $head(\R)=p$ and $body(\R)=\{B_1, \dots, B_n\}$. When $n=0$, the  derivation antecedent $\pi(B_1) \ \dotsc \ \pi(B_n)$ is replaced by $\top$ (corresponding to the empty body).\qed
\end{definition}

\noindent
Each derivation in a proof is a particular application of Modus Ponens where, once the body (conjunction of literals $B$) of a rule $\R$ ($p\leftarrow B$) has been proved, then the head $p$ can be concluded.

\begin{figure}[htbp]
\vspace{-20pt}
\[
\cfrac{
	\cfrac{
	   \cfrac{\top}{drive} \ (d) \hspace{20pt}
	   \cfrac{\top}{drunk} \ (k)
	}
	{punish} \ (l)
}
{prison} \ (e)
\hspace{20pt}
\cfrac{
  \cfrac{
     \cfrac{\top}{resist} \ (r)
  }
  {punish} \ (m)
} 
{prison} \ (e)
\hspace{20pt}
\cfrac {
	\cfrac{
		\cfrac{
			\cfrac{
			   \cfrac{\top}{drive} \ (d) \hspace{20pt}
			   \cfrac{\top}{drunk} \ (k)
			}
			{punish} \ (l)
		}
		{sentence} \ (s)
	}
	{punish} \ (n)
}
{prison} \ (e)
\]
\caption{Some proofs for atom $prison$ (the rightmost proof is redundant).}
\label{fig:proofs}
\vspace{-20pt}
\end{figure}
\begin{examplecont}{ex:prison}
\noindent Program \progref{prg:prison} is positive and, in fact, completely labelled, so we can identify each rule with its label. Atom $prison$ can be derived in \progref{prg:prison} using the two proofs on the left in Figure~\ref{fig:proofs}. These two proofs have a clear correspondence to causes $(d*k)\cdot l e$ and $r m e$ depicted in Figure~\ref{fig:prisonB}. In fact, the least model $I$ of \progref{prg:prison} assigns causal value $I(punish)=(d*k) \cdot l e + r m e$.\qed
\end{examplecont}

\noindent Let $P$ be a positive, completely labelled program. Given a proof $\pi$, we define its graph $G_\pi$ as follows. For each sub-derivation in $\pi$ of the form $\pi(p)$ in Definition~\ref{def:proof} we include an edge $(l_i,m)$ where $m$ is the label of rule $R$ and $l_i$ is the label of the top-level rule in $\pi(B_i)$, for all $i=1,\dots,n$. The vertices in $G_\pi$ exclusively collect the labels in those edges. We define $graph(\pi) \ \eqdef \ G^*_\pi$. The two left proofs in Figure~\ref{fig:proofs} are then obviously mapped to the causal graphs in Figure~\ref{fig:prisonB}. If $\Pi$ is a set of proofs, we define
$graph(\Pi) \ \eqdef \  \{graph(\pi) \mid \pi \in \Pi \}$.

A proof can be sometimes redundant, in the sense that some of its derivation steps could be removed. A natural way of defining a non-redundant proof is resorting to its associated graph. We say that a proof $\pi(p)$ of an atom $p$ in a positive, completely labelled program $P$ is \emph{redundant} if there exists another proof of $p$, $\pi'(p)$, such that $graph(\pi(p)) \leq graph(\pi'(p))$, in other words, we can build another proof $\pi'$ with a smaller associated graph.

\begin{example}\label{ex:prison3}
Suppose that we introduce an atom $sentence$ which acts as a synonym for $punish$. Furthermore, assume law $m$ mentions $sentence$ as its head now, instead of $punish$. Hence, let \newprog\label{prg:prison3} be program:
\vspace{-4pt}
\[
\begin{array}{rl@{\hspace{25pt}}rl@{\hspace{25pt}}rl}
l: & punish \leftarrow drive, drunk  & d: & drive & n: & punish \leftarrow sentence\\
m: & sentence \leftarrow resist      & k: & drunk & s: & sentence \leftarrow punish\\
e: & prison \leftarrow punish        & r: & resist
\end{array}
\]
\vspace{-8pt}

\noindent Then, the rightmost proof shown in Figure~\ref{fig:proofs} together with its associated graph $(d*k) \cdot l s n e$ is redundant, since the (still valid) leftmost proof in Figure~\ref{fig:proofs} for $prison$ has an associated stronger cause (or smaller graph) $(d*k) \cdot l e$. Considering the positive loop formed by $n$ and $s$, one may wonder why it does not spoil the computation of $T_{\progref{prg:prison3}}$ to iterate forever (adding more and more concatenations of $n$ and $s$). The reason is that, at a given point, subsequent iterations yield redundant graphs subsumed by previous steps. In particular, the iteration of 
$T_{\progref{prg:prison3}}$ yields the steps:
\[
\begin{array}{rccccccc}
i & drive & drunk & resist & sentence & punish & prison\vspace{-5pt}\\
\hline\vspace{-15pt}\\
1 & d     & k     & r      & 0                & 0                & 0      \\ 
2 & d     & k     & r      & rm               & (d*k)\cdot l     & 0      \\ 
3 & d     & k     & r      & rm+(d*k)\cdot ls & (d*k)\cdot l+rmn & (d*k)\cdot le    \\ 
4 & d     & k     & r      & rm+(d*k)\cdot ls & (d*k)\cdot l+rmn & ((d*k)\cdot l+rmn)\cdot e 
\end{array}
\]
\noindent reaching a fixpoint at step $4$. The value for $sentence$ at step $4$ would actually be the sum of $rm$ (derived from $resist$) with the value of $punish$ in the previous step, $(d*k)\cdot l+rmn$ followed by~$s$. This corresponds to:
\begin{IEEEeqnarray*}{rCl?l}
rm + (\underbrace{(d*k)\cdot l+rmn}_{punish}) \cdot s &=& rm + (d*k)\cdot ls + rmns & \text{distributivity}\vspace{-0.5cm}\\
 &=& rm + rmns + (d*k)\cdot ls & \text{commutativity} \\
 &=& rm + rm\cdot ns + (d*k)\cdot ls & \text{associativity} \\
 &=& rm + 1 \cdot rm \cdot ns + (d*k)\cdot ls & \text{identity} \\
 &=& rm + (d*k)\cdot ls & \text{absorption for `$+$' and `$\cdot$'} 
\end{IEEEeqnarray*}

\noindent That is, iterating the loop $rmns$ is redundant since a stronger cause $rm$ was obtained before.\qed
\end{example}

\vspace{-8pt}
\begin{theorem}\label{th:proofs}
Let $P$ be a positive, completely labelled program, and $\Pi_p$ the set of non-redundant proofs of some atom $p$ with respect to $P$. If $I$ denotes the least model of $P$, then:
\begin{IEEEeqnarray*}{c+x*}
G \in graph(\Pi_p) \ \ \ \text{iff} \ \ \ G \ \text{is a maximal causal graph in } I(p) & $\Box$
\end{IEEEeqnarray*}
\end{theorem}

\noindent Note the importance of this result: it reveals that the information we obtain by a purely semantic treatment of causal values (computing the least model by algebraic operations) has a one-to-one correspondence to syntactic proofs obtained by modus ponens that are further guaranteed to be non-redundant (they do not contain unnecessary steps).
Completely labelled programs are interesting for establishing the correspondence in the theorem above, but there are several scenarios in which one may be interested in disregarding the effect of rules in a program or in identifying a group of rules under the same label.

\begin{example}\label{ex:prison4}~Let \newprog\label{prg:prison4} be the following variation of \progref{prg:prison3}:
\[
\begin{array}{rl@{\hspace{25pt}}rl@{\hspace{25pt}}rl}
z: & sentence \leftarrow drive, drunk  & d: & drive & punish \leftarrow sentence\\
z: & punish \leftarrow resist          & k: & drunk & sentence \leftarrow punish\\
e: & prison \leftarrow punish          & r: & resist
\end{array}
\]
\noindent where $l$ and $m$ in \progref{prg:prison3} are now just two cases of a common law $z$, and $punish$ and $sentence$ depend on each other through unlabelled rules.\qed
\end{example}

\noindent Removing the labels in the positive cycle between $sentence$ and $punish$ captures the idea that, since they are synonyms, whenever we have a cause for $sentence$, it immediately becomes a cause for $punish$ and vice versa. By iterating the $T_P$ operator, it is not difficult to see that the least causal model $I_{\ref{prg:prison4}}$ makes the assignments $I_{\ref{prg:prison4}}(sentence)=I_{\ref{prg:prison4}}(punish)=(d*k)\cdot z+rz$ (that is $sentence$ and $punish$ are equivalent) and $I_{\ref{prg:prison4}}(prison)=(d*k)\cdot ze + rze$. This result could also be computed from the least model $I_{\ref{prg:prison3}}$ for \progref{prg:prison3} by replacing $l$ and $m$ by $z$ and ``removing'' $n$ and $s$ (that is, replacing them by $1$). This is, in fact, a general property we formalise as follows. Given two causal terms $t,u$ and a label $l$, we define $t[l\mapsto u]$ as the result of replacing label $l$ in $t$ by term $u$.

\begin{theorem}\label{theorem:least.model.label.replacing}
Let $P$ be a positive causal logic program and $P'$ be the result of replacing a label $l$ in $P$ by some $u$, where $u$ is any label or $1$. Furthermore, let $I$ and $I'$ be the least models of $P$ and $P'$, respectively.
Then, $I'(p)=I(p)[l \mapsto u]$ for any atom $p$.\qed
\end{theorem}

\noindent In particular, in our example, $I_{\ref{prg:prison4}}(p) = I_{\ref{prg:prison3}}(p) [l\mapsto z][m \mapsto z][n \mapsto 1][s \mapsto 1]$, for any atom $p$.
If we remove all labels in a program, we eventually get a standard, unlabelled program. Obviously, its least model will be two-valued, since removing all labels in causal terms, eventually collapses all of them to $\{0,1\}$. As a result, we can easily establish the following correspondence.

\begin{theorem}\label{theorem:least.model.classical.correspondence}
Let $P$ be a causal positive logic program and $P'$ its unlabelled version. Furthermore, let $I$ be the least causal model of $P$ and $I'$ the least classical model of $P'$. Then $I'=I^{cl}$. \qed
\end{theorem}

\section{Default negation} \label{sec:negation}

To introduce default negation, let us consider the following variation of our running example.

\begin{example}\label{ex:prison2}
Assume now that law $e$ is a default and that there may be exceptional cases in which punishment is not effective. In particular, some of such exceptions are a pardon, that the punishment was revoked, or that the person has diplomatic immunity.
A possible program \newprog\label{prg:prison2} encoding this variant of the scenario is:
\[
\begin{array}{rl@{\hspace{25pt}}rl@{\hspace{25pt}}l}
l: & punish \leftarrow drive, drunk          & d: & drive & 
    abnormal \leftarrow pardon\\
m: & punish \leftarrow resist                & k: & drunk & 
    abnormal \leftarrow revoke \\
e: & prison \leftarrow punish, \Not abnormal & r: & resist &
    abnormal \leftarrow diplomat
\end{array}
\]
This program has a unique stable model which still keeps $prison$ true, since \emph{no proof for} $abnormal$ could be obtained, i.e. no exception occurred.
\qed
\end{example}

\noindent From a causal perspective, saying that the lack of an exception is part of a cause (e.g., for imprisonment) is rather counterintuitive. It is not the case that we go to prison because of not receiving a pardon, not having a punishment revocation, not being a diplomat, or whatever possible exception that might be added in the future\footnote{A case of the well-known \emph{qualification problem}~\cite{McC77}, i.e., the impossibility of listing all the possible conditions that prevent an action to cause a given effect. Appendix B (available online) contains a more elaborated example showing how the qualification problem may affect causal explanations when inertia is involved.}. Instead, as nothing violated default $e$, the justifications for $prison$ should be those shown in Figure~\ref{fig:prisonA}. In this way, falsity becomes the \emph{default situation} that is broken when a cause is found\footnote{The paper~\cite{Hitchcock2009cause} contains an extended discussion with several examples showing how people ordinarily understand causes as deviations from a norm.}. This interpretation carries over to negative literals, so that the presence of $\Not p$ in a rule body does not propagate causal information, but instead is a check for the absence of an exception. To capture this behaviour, we proceed to extend the traditional program reduct \cite{GL88} to causal logic programs.

\begin{samepage}
\begin{definition}[Program reduct]
The \emph{reduct} of program $P$ with respect to causal interpretation $I$, in symbols $P^I$, is the result of:
\vspace{-5pt}
\begin{enumerate}
\item removing from $P$ all rules $\R$, s.t.~$I(B) \neq 0$ for some negative literal $B\in\mathit{body}(\R)$;
\item removing all negative literals from the remaining rules of $P$.\qed
\end{enumerate}
\end{definition}
\end{samepage}

\noindent
An interpretation $I$ is a \emph{causal stable model} of program $P$ iff $I$ is the least causal model of $P^I$.

\begin{examplecont}{ex:prison2}
Suppose that we add atoms $(p: pardon)$ and $(d: diplomat)$ to program \progref{prg:prison2}. The only stable model $I$ of this extended program makes $I(prison)=0$ and $I(abnormal)=p+d$ as expected.\qed
\end{examplecont}

\begin{theorem}[Correspondence to non-causal stable models]\label{th:csm}
Let $P$ be a causal logic program and $P'$ its unlabelled version. Then:
\vspace{-5pt}
\begin{enumerate}
\item If $I$ is a causal stable model of $P$, then $I^{cl}$ is a stable model of $P'$.
\item If $I'$ is a stable model of $P'$ then there is a unique causal stable model $I$ of $P$ s.t. $I'=I^{cl}$.\qed
\end{enumerate}
\end{theorem}

\noindent
This theorem also shows a possible method for computing causal stable models of a program $P$. We may first run a standard ASP solver on the unlabelled version of $P$ to obtain a stable model~$I'$.
This stable model~$I'$ has a corresponding causal stable model $I$, such that $I'=I^{cl}$ and both interpretations coincide in their assignment of $0$'s. Therefore, $P^I=P^{I'}$ and we can use the latter to iterate the $T_P$ operator and obtain the least causal model of this reduct, which will mandatorily be a causal stable model due to Theorem~\ref{th:csm}.

\section{Related Work} \label{sec:related}

\citeN{Cabalar11} already introduced the main motivations of our work, but used \emph{ad hoc} operations on proof trees without resorting to algebraic structures. A preliminary version~\cite{CF13} of the current approach relied on chains of labels but was actually \emph{weaker}, missing basic properties we can derive now from causal graphs. 

There exists a vast literature on causal reasoning in Artificial Intelligence.
Papers on reasoning about actions and change~\cite{Lin95,NMc97,Thi97} have been traditionally focused on using causal inference to solve representational problems (mostly, the frame, ramification and qualification problems) without paying much attention to the derivation of cause-effect relations. Perhaps the most established AI approach for causality is relying on \emph{causal networks}~\cite{Pearl00,HP05,Hal08}. In this approach, it is possible to conclude cause-effect relations like ``$A$ has caused $B$'' from the behaviour of structural equations by applying the counterfactual interpretation from \citeN{Hum1748}: ``had $A$ not happened, $B$ would not have happened.'' As discussed by~\citeN{Hall04}, the counterfactual-based definition of causation corresponds to recognising some kind of \emph{dependence} relation in the behaviour of a non-causal system description. As opposed to this, Hall considers a different (and incompatible) definition where causes must be connected to their effects via \emph{sequences of causal intermediates}, something that is closer to our explanations in terms of causal graphs.

Apart from the different AI approaches and attitudes towards causality, from the technical point of view, the current approach can be classified as a \emph{labelled deductive system}~\cite{BGLR04}. In particular, the work that has had a clearest and most influential relation to the current proposal is the \emph{Logic of Proofs} (\LP) by~\citeN{Art01}. We have borrowed from that formalism part of the notation for our causal terms and rule labellings and the fundamental idea of keeping track of justifications by considering  rule applications. 

Focusing on LP, our work obviously relates to explanations as provided by approaches to debugging in ASP~\cite{GPST08,PSE09,SST13,DAA13}.
\citeN{PAA91} and \citeN{Denecker1993Justification} also define different semantics in terms of justifications, but do not provide calculi for them.
In these works, explanations usually contain all possible ways to derive an atom or to prevent its derivation, including paths through negation. This differs from a KR orientation where only the cause-effect relations that ``break the norm'' should be considered relevant. This point of view is also shared, e.g., by the counterfactual-based causal LP approach~\mbox{\cite{Vennekens11}}. \citeN{Fages1991} characterised stable models in terms of loop-free justifications expressed as partial order relations among atoms in positive bodies. We conjecture that the causal values obtained in our semantics formally capture Fages' justifications.
%
%
A more far-fetched resemblance exists to work on the analysis of tabled Prolog computations. There, the goal is to identify potential causes for non-termination of program evaluations, which can be achieved examining so-called \emph{forest logs}, i.e., a log of table operations for a computation. By adding unique labels for rules (with the original intention to disambiguate analysis results, cf.~\citeN{lian-kife-13}, however not as an explicit means for representing knowledge), in principle a forest log implicitly contains the information necessary to read of the causal model of a completely labelled positive causal logic program. 

\section{Conclusions} \label{sec:conc}

In this paper we have provided a multi-valued semantics for normal logic programs whose truth values form a lattice of causal graphs.
A causal graph is nothing else but a graph of rule labels that reflects some order of rule applications. In this way, a model assigns to each true atom a value that contains justifications for its derivation from the existing rules. We have further provided three basic operations on the lattice: an addition, that stands for alternative, independent justifications; a product, that represents joint interaction of causes; and a concatenation that reflects rule application. We have shown that, for positive programs, there exists a least model that coincides with the least fixpoint of a direct consequences operator, analogous to~\mbox{\citeN{vEK76}}. With this, we are able to prove a direct correspondence between the semantic values we obtain and the syntactic idea of proof. These results have been extrapolated to stable models of programs with default negation, understanding the latter as ``absence of cause.''
Although, for space reasons, we have not dealt with programs with variables, their semantics is obtained from their (possibly infinite) grounding, as usual.

Several topics remain open for future study. An interesting issue is to replace the syntactic definition by a reduct in favour of a logical treatment of default negation, as has been done for (non-causal) stable models and their characterisation in terms of Equilibrium Logic \cite{Pearce06}. 
Regarding the representation of causal information, a natural next step would be the consideration of syntactic operators for more specific knowledge like the influence of a particular event or label in a conclusion, expressing necessary or sufficient causes, or even dealing with counterfactuals.
Further ongoing work is focused on implementation, complexity assessment, and an extension to disjunctive programs, respectively the introduction of strong negation. 
Exploring related areas of KR and reasoning, such as, e.g., Paraconsistent Reasoning and Belief Revision, seems promising with respect to extending the range of problems to which our approach may effectively be applied.


\paragraph{Acknowledgements} We are thankful to David Pearce, Manuel Ojeda, Jes\'us Medina, Carlos Damasio and Joost Vennekens for their suggestions and comments on earlier versions of this work. We also thank the anonymous reviewers for their help to improve the paper.

\newpage
\bibliographystyle{acmtrans}
\bibliography{refs}

\newpage

\section*{Appendix A. Auxiliary figures}\label{sec:figs}

\begin{figure}[htbp]
\begin{center}
$
\begin{array}{c}
\hbox{\em Associativity} \\
\hline
\begin{array}{r@{\ }c@{\ }r@{}c@{}l c r@{}c@{}l@{\ }c@{\ }l@{\ }}
t & + & (u & + & w) & = & (t & + & u) & + & w\\
t & * & (u & * & w) & = & (t & * & u) & * & w
\end{array}
\end{array}
$
\ \
$
\begin{array}{c}
\ \ \ \ \hbox{\em Commutativity}\ \ \ \ \\
\hline
\begin{array}{r@{\ }c@{\ }l c r@{\ }c@{\ }l@{\ }}
t & + & u & = & u & + & t\\ 
t & * & u & = & u & * & t
\end{array}
\end{array}
$
\ \
$
\begin{array}{c}
\hbox{\em Absorption} \\
\hline
\begin{array}{c c r@{\ }c@{\ }r@{}c@{}l@{\ }}
t & = & t & + & (t & * & u)\\
t & = & t & * & (t & + & u)
\end{array}
\end{array}
$
\ \
\\
\vspace{10pt}
$
\begin{array}{c}
\hbox{\em Distributive} \\
\hline
\begin{array}{r@{\ }c@{\ }r@{}c@{}l c r@{}c@{}l@{\ }c@{\ }r@{}c@{}l@{}}
t & + & (u & * & w) & = & (t & + & u) & * & (t & + & w)\\
t & * & (u & + & w) & = & (t & * & u) & + & (t & * & w)
\end{array}
\end{array}
$
\ \
$
\begin{array}{c}
Identity \\
\hline
\begin{array}{rcr@{\ }c@{\ }l@{\ }}
t & = & t & + & 0\\
t & = & t & * & 1
\end{array}
\end{array}
$
\ \
$
\begin{array}{c}
\hbox{\em Idempotence} \\
\hline
\begin{array}{rcr@{\ }c@{\ }l@{\ }}
t & = & t & + & t\\
t & = & t & * & t
\end{array}
\end{array}
$
\ \ 
$
\begin{array}{c}
\hbox{\em Annihilator} \\
\hline
\begin{array}{rcr@{\ }c@{\ }l@{\ }}
1 & = & 1 & + & t\\
0 & = & 0 & * & t
\end{array}
\end{array}
$
\end{center}
\caption{Sum and product satisfy the properties of a completely distributive lattice.}
\label{fig:DBLattice}
\end{figure}

\section*{Appendix B. An example of causal action theory}\label{sec:action}

In this section we consider a more elaborated example from~\citeN{Pearl00}.

\begin{example}\label{ex:sw}
Consider the circuit in Figure~\ref{fig:sw} with two switches, $a$ and $b$, and a lamp $l$. Note that $a$ is the main switch, while $b$ only affects the lamp when $a$ is up. Additionally, when the light is on, we want to track which wire section, $v$ or $w$, is conducting current to the lamp.\qed
\end{example}

As commented by~\citeN{Pearl00}, the interesting feature of this circuit is that, seen from outside as a black box, it behaves exactly as a pair of independent, parallel switches, so it is impossible to detect the causal dependence between $a$ and $b$ by a mere observation of performed actions and their effects on the lamp.  Figure~\ref{fig:sw} also includes a possible representation for this scenario; let us call it program \newprog\label{prg:sw}. It uses a pair of fluents $up(X)$ and $down(X)$ for the position of switch $X$, as well as $on$ and $\fn{off}$ to represent the state of the lamp. Fluents $up(X)$ and $down(X)$ (respectively, $on$ and $\fn{off}$) can be seen as the strong negation of each other, although we do not use an operator for that purpose\footnote{Notice how strong negation would point out the cause(s) for a boolean fluent to take value false, whereas default negation represents the absence of cause.}. Action $m(X,D)$ stands for ``move switch $X$ in direction $D \in \{u,d\}$'' ($up$ and $down$, respectively). Actions between state $t$ and $t+1$ are located in the resulting state. Finally, we have also labelled inertia laws (by $i$) to help keeping track of fluent justifications inherited by persistence.

\begin{figure}[htbp]
\begin{center}
\[
\begin{array}{cc}
\begin{array}{c}
\hbox{\underline{Circuit diagram}} \\ \\
\includegraphics[scale=0.55]{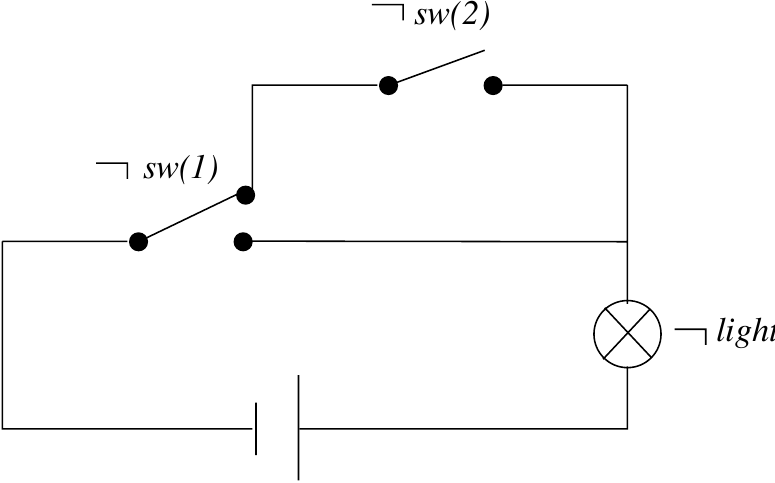} \\
 \\
\hbox{\underline{Inertia}}\\
\begin{array}{rrcl}
i: & up(X)_{t+1}   & \leftarrow & up(X)_t, \ \Not down(X)_{t+1}\\
i: & down(X)_{t+1}   & \leftarrow & down(X)_t, \ \Not up(X)_{t+1}
\end{array}
\end{array}
&
\begin{array}{c}
\hbox{\underline{Direct effects}} \\
\begin{array}{rrcl}
up(X)_t   & \leftarrow & m(X,u)_t\\
down(X)_t & \leftarrow & m(X,d)_t
\end{array}\\ \\
\hbox{\underline{Causal rules (indirect effects)}}\\
\begin{array}{rrcl}
v: & on_t & \leftarrow & down(a)_t\\
w: & on_t & \leftarrow & up(a)_t, \ down(b)_t \\
 & \fn{off}_t & \leftarrow & up(a)_t, \ up(b)_t   
\end{array}\\ \\
\hbox{\underline{Constraints}}\\
\begin{array}{rcl}
\bot & \leftarrow & m(X,u)_t, \ m(X,d)_t \\
\bot & \leftarrow & up(X)_t, \ down(X)_t \\
\bot & \leftarrow & on_t, \ \fn{off}_t
\end{array}
\\ \\
\hbox{\underline{Initial state}}\\
up(a)_0 \hspace{15pt} up(b)_0 \hspace{15pt} \fn{off}_0
\end{array} 
\end{array}
\]

\end{center}
\caption{A circuit with two switches together with a possible representation.}
\label{fig:sw}
\end{figure}

Suppose we perform the following sequence of actions: we first move down both switches, next switch $b$ is moved first up and then down, and finally we move up switch $a$. Assume also that each action occurrence is labelled with the action name so that, for instance, moving $b$ up in Situation~$1$ corresponds to the program fact $m(b,u)_1 : m(b,u)_1$. The table in Figure~\ref{fig:swproject} shows the resulting temporal projection. Note how the lamp turns on in Situation~$1$ but only because of $v$, that is, moving $a$ down. Movements of $b$ at $2$ and $3$ do not affect the lamp, and its  causal explanation ($down(a)$) is maintained by inertia. In Situation~$4$, the lamp is still on but the reason has changed. The explanation this time is that we had closed down $b$ at $3$ (and this persisted by inertia) while we have just moved $a$ up, firing rule $w$. 

\begin{figure}[htbp]
\begin{center}
\[
\begin{array}{c|ccccc}
t & 0 & 1 & 2 & 3 & 4\\
\hline
\hbox{Actions} & & m(a,d)_1, \ m(b,d)_1 & m(b,u)_2 & m(b,d)_3 & m(a,u)_4 \\
\hline
up(a)_t   & 1 & 0        & 0                & 0                & m(a,u)_4 \\
down(a)_t & 0 & m(a,d)_1 & m(a,d)_1 \cdot i & m(a,d)_1 \cdot i & 0 \\
\hline
up(b)_t   & 1 & 0        & m(b,u)_2  & 0        & 0 \\
down(b)_t & 0 & m(b,d)_1 & 0         & m(b,d)_3 & m(b,d)_3 \cdot i \\
\hline
on_t       & 0 & m(a,d)_1 \cdot v & m(a,d)_1 \cdot iv & m(a,d)_1 \cdot iv & 
  (\ m(b,d)_3 \cdot i * m(a,u)_4 \ )\cdot w \\
\fn{off}_t & 1 & 0                & 0                 & 0                 & 0 \\
\end{array}
\]
\end{center}
\caption{Temporal projection of a sequence of actions for program \progref{prg:sw}.}
\label{fig:swproject}
\end{figure}

\noindent This example also illustrates why we are not interested in providing negative justifications through default negation. This would mean to explicitly include non-occurrences of actions that might otherwise have violated  inertia. For instance, the explanation for $on_2$ would include the fact that we did not perform $m(a,u)_2$. Including this information for one transition is perhaps not so cumbersome, but suppose that, from $2$ we executed a high number of transitions without performing any action. The explanation for $on_3$ would \emph{additionally} collect that we did not perform $m(a,u)_3$ either. The explanation for $on_4$ should also collect the negation of further  possibilities: moving $a$ up at $4$; three movements of $a$ up, down and up; moving $b$ at $3$ and both switches at $4$; moving both switches at $3$ and $b$ at $4$; etc. It is easy to see that negative explanations grow exponentially: at step $t$ we would get the negation of \emph{all possible plans} for making $on_t$ false, while indeed, \emph{nothing has actually happened} (everything persisted by inertia).

\begin{example}[The gear wheels]\label{ex:wheels}
Consider a gear mechanism with a pair of wheels, each one powered by a separate motor. Each motor has a switch to start and stop it. There is another switch to connect or disconnect the wheels. \cite{McCain1997}.\qed
\end{example}

\noindent
This example can be captured by the logic program \newprog\label{prg:gears} formed by the following causal rules:

\vspace{-0.25cm}

{\small
\begin{gather*}
	\begin{IEEEeqnarraybox}{rlClCl}
			  &motor(W)_t	&\leftarrow&	start(W)_t				\\
		\sneg &motor(W)_t	&\leftarrow&	stop(W)_t				\\
			  &turn(W)_t	&\leftarrow&	motor(W)_t				\\
			  &coupled_t	&\leftarrow&	couple_t				\\
		\sneg &coupled_t	&\leftarrow&	uncouple_t 
	\end{IEEEeqnarraybox}
	\hspace{2cm}
	\begin{IEEEeqnarraybox}{rlCcl}
			  &turn(1)_t	&\leftarrow&	&turn(2)_t,\ coupled_t	\\
			  &turn(2)_t	&\leftarrow&	&turn(1)_t,\ coupled_t	\\
		\sneg &turn(1)_t	&\leftarrow& \sneg &turn(2)_t,\ coupled_t\\
		\sneg &turn(2)_t	&\leftarrow& \sneg &turn(1)_t,\ coupled_t\\
	\end{IEEEeqnarraybox}
\end{gather*}}

\vspace{-0.25cm}

plus the fhe following inertia axioms:

\vspace{-0.25cm}

{\small
\begin{IEEEeqnarray}{rl C rl rrl}
	  &motor(W)_{t+1} &\leftarrow& 		 &motor(W)_t,\
										&\Not& \sneg  &motor(W)_{t+1}
		\notag\\
\sneg &motor(W)_{t+1} &\leftarrow& \sneg &motor(W)_t,\
										&\Not& 		&motor(W)_{t+1}
		\notag\\
	  &turn(W)_{t+1} &\leftarrow& 		 &turn(W)_t,\
										&\Not& \sneg  &turn(W)_{t+1}
		\label{prg:gears.inertia.turn}\\
\sneg &turn(W)_{t+1} &\leftarrow& \sneg &turn(W)_t,\
										&\Not& 		&turn(W)_{t+1}
		\notag\\
	  &coupled_{t+1} &\leftarrow& 		 &coupled_t,\
										&\Not& \sneg  &coupled_{t+1}
		\notag\\
\sneg &coupled_{t+1} &\leftarrow& \sneg &coupled_t,\
										&\Not& 		&coupled_{t+1}
		\notag
\end{IEEEeqnarray}}

\vspace{-0.25cm}

\noindent
Suppose that initially both motors are off an the wheels are immobile. This is reflected by the following set of facts $\setb{ \!\!\sneg motor(1)_0, \sneg motor(2)_0, \sneg turn(1)_0, \sneg turn(2)_0, \sneg cuopled_0}$. Then we perform the following actions $\setb{s(1)_3 : start(1)_3, \ \ c_6 : couple_6, \ \ u_9 : uncouple_9 }$. Figure~\ref{fig:gears.project.A} shows the causal values associated with each fluent in each time interval. Note that we only have labelled the actions avoiding tracing rule application for clearity sake. This example illustrates the behaviour of causal logic programs in the presence of causal cycles. When no action is performed the value of a fluent is just propagate by inertia.

An iteresting variation of this example is incorporating some mechanic device to stop the wheels \cite{deneker1998,lin2011}. This can be achived by adding the following set of rules:

\vspace{-0.25cm}

{\small
\begin{gather*}
	\begin{IEEEeqnarraybox}{rlClCl}
			  &breaked(W)_t	&\leftarrow&	break(W)_t				\\
		\sneg &breaked(W)_t	&\leftarrow&	unbreak(W)_t			\\
		\sneg &turn(W)_t	&\leftarrow&	breaked(W)_t		
	\end{IEEEeqnarraybox}
	\hspace{1.5cm}
	\begin{IEEEeqnarraybox}{rl C rl rrl}
	  &breaked(W)_{t+1} &\leftarrow& 		 &breaked(W)_t,\
										&\Not& \sneg  &breaked(W)_{t+1}
		\notag\\
\sneg &breaked(W)_{t+1} &\leftarrow& \sneg &breaked(W)_t,\
										&\Not& 		&breaked(W)_{t+1}
		\notag\\
	\end{IEEEeqnarraybox}
\end{gather*}}

\vspace{-0.25cm}

\noindent If the second wheel is breaked at $s_9$ we will get insted that $turn(2)_9$ is false an the cause of $\sneg turn(2)_t$ in any future situation is $break(2)_9$. Another interesiting situations is when we break the second wheel in the situation $5$. In such case we have an incosistence. On the one hand the cause of $turn(1)_6$ and $turn(2)_6$ are respectively $s_1$ and $s(1)_3*c(6)$. On the other hand the cause of $\sneg turn(1)_6$ and $\sneg turn(2)_6$ are respectively $break(2)_5*c(6)$ and $break(2)_5$. Note that, causal values not only points the existence of an incosistence but also can be explaoited to explain why. In particual, in this case, the motor forces the wheels to spin whereas the break opposed to it.

\begin{figure}\centering
\subfigure[Innertial turning]{ \label{fig:gears.project.A}
$
\begin{array}{r|ccccc}
t & 0-2 & 3-5 & 6-8 & 9-&\\
\hline
motor(1)_t	& 0 & s(1)_3  & s(1)_3 		  & s(1)_3					\\
\sneg motor(1)_t	& 1 & 0   & 0 			  & 0 					\\
\hline
coupled_t		& 0 & 0   & c_6 		  & 0						\\
\sneg coupled_t	& 1 & 1   & 0 			  & u_9					 	\\
\hline
turn(1)_t	& 0 & s(1)_3  & s(1)_3 		  & s(1)_3					\\
\sneg turn(1)_t	& 1 & 0   & 0 			  & 0 						\\
\hline
turn(2)_t		& 0 & 0	  & s(1)_3 * c_6  & s(1)_3 * c_6			\\
\sneg turn(2)_t	& 1 & 1	  & 0 			  & 0 						\\
\end{array}
$
}
\subfigure[Friction brake]{ \label{fig:gears.project.B}
$
\begin{array}{r|ccccc}
t & 0-8 & 9-&\\
\hline
motor(1)_t			& * & * \\
\sneg motor(1)_t	& * & * \\
\hline
coupled_t			& * & * \\
\sneg coupled_t		& * & * \\
\hline
turn(1)_t			& * & * \\
\sneg turn(1)_t		& * & * \\
\hline
turn(2)_t			& * & 0		\\
\sneg turn(2)_t		& * & f(2)_9 	\\				
\end{array}
$
}
\caption{Temporal projection of a sequence of actions for program \progref{prg:gears}.}
\end{figure}

Another iteresting variation of this example is assuming that wheels are breaked by friction as soon as the motriz force over them is absent. This can be achived by just replacing the inertia axiom \eqref{prg:gears.inertia.turn} for $turn$ by the following causal rule:

\vspace{-0.25cm}

{\small
\begin{IEEEeqnarray*}{rl C rl rrl}
f(W)_{t} : \sneg &turn(W)_{t+1} &\leftarrow& \Not \sneg &turn(W)_t,\
										&\Not& 		&turn(W)_{t+1}
\end{IEEEeqnarray*}}

\vspace{-0.25cm}

\noindent
Figure~\ref{fig:gears.project.B} shows the variation in the causal value of $turn(2)$. Note that $*$ means that the value is the same as Figure~\ref{fig:gears.project.A}.

\newpage
\section*{Appendix C. Example with infinite rules}\label{sec:infinite}

%

\begin{example}
Consider the infinite program $\newprog\label{prog:infinite}$ given by the ground instances of the set of rules:
\begin{IEEEeqnarray*}{c C ; l C l}
l(s(X))  & : &		nat(s(X)) & \leftarrow & nat(X)
\\
l(z) 	 & : &		nat(z)
\end{IEEEeqnarray*}
\noindent defining the natural numbers with a Peano-like representation, where $z$ stands for ``zero.'' For each natural number $n$, the causal value obtained for $nat(s^n(z))$ in the least model of the program is $l(z)\cdot l(s(z)) \dotsc l(s^n(z))$. Read from right to left, this value can be seen as the computation steps performed by a top-down Prolog interpreter when solving the query $nat(s^n(z))$.
As a further elaboration, assume that we want to check that at least some natural number exists. For that purpose, we add the following rule to the previous program:

\vspace{-17pt}

\begin{gather}
	some \leftarrow nat(X) \label{f:some}
\end{gather}

\vspace{-2pt}

\noindent The interesting feature of this example is that atom $some$ collects an infinite number of causes from all atoms $nat(s^n(z))$ with $n$ ranging among all the natural numbers. That is, the value for $some$ is $I(some)=\alpha_0+\alpha_1+\alpha_2+\dotsc+\alpha_n+\dots$ where $\alpha_n\eqdef l(z)\cdot l(s(z)) \cdot\dotsc\cdot l(s^n(z))$. However, it is easy to see that the fact $nat(z)$ labelled with $l(z)$ is not only sufficient to prove the existence of some natural number, but, due to the recursive definition of the natural numbers, it is also \emph{necessary} -- note how all the proofs $\alpha_i$ actually begin with an application of $l(z)$.

This fact is captured in our semantic by the algebraic equivalences showed in Fig~\ref{fig:appl}. From associativity and identity of `$\cdot$', the following equivalence holds:
\begin{gather*}
\alpha_n = 1 \ \cdot \ l(z) \ \cdot \ \beta_n
	\hspace{1cm}\text{with }\hspace{0.5cm}
	\beta_i\eqdef l(s(z)) \ \cdot \ \dotsc \ \cdot \ l(s^i(z))
\end{gather*}
for any $n\geq 1$. Furthermore, from absorption of `$\cdot$' w.r.t the addition, it also holds that
\begin{gather*}
\alpha_0 + \alpha_n \ \ \ \ =\ \ \ \ l(z) 
					\ +\ 1 \ \cdot \ l(z) \ \cdot \ \beta_n
					\ \ \ \ =\ \ \ \ l(z)
\end{gather*}
As a consequence, the previous infinite sum just collapses to $I(some)=l(z)$, reflecting the fact that, to prove the existence of a natural number, only the fact labelled as $l(z)$ is relevant.

Suppose now that, rather than defining natural numbers in a recursive way, we define them by directly asserting an infinite set of facts as follows:
\begin{gather}
\begin{IEEEeqnarraybox}{r C ; l C l}
l(s^n(z))  & : &		nat(s^n(z))
\end{IEEEeqnarraybox} \label{f:inffacts}
\end{gather}
\noindent for any $n\geq 0$, where $s^0(z)$ stands for $z$. In this variant, the causal value obtained for $nat(s^n(z))$ is simply $l(s^n(z))$, so that the dependence we had before on lower natural numbers does not exist any more. Adding rule \eqref{f:some} to the set of facts \eqref{f:inffacts} allows us concluding $I(some)=l(z)+l(s(z))+l(s^2(z))+\dotsc+l(s^n(z))+\dots$ and this infinite sum cannot be collapsed into any finite term. This reflects that we have now infinite \emph{independent} ways to prove that some natural number exists.

This last elaboration can be more elegantly captured by replacing the infinite set of facts \eqref{f:inffacts} by an auxiliary recursive predicate $aux$ defined as follows:
\begin{gather*}
\begin{IEEEeqnarraybox}{r C ; l C l}
&&		aux(s(X)) & \leftarrow & aux(X)
\\
&&		aux(z)
\end{IEEEeqnarraybox}
\hspace{2cm}
\begin{IEEEeqnarraybox}{l C ; l C l}
l(X) &:&	 nat(X)	& \leftarrow & aux(X)\\
\end{IEEEeqnarraybox}
\end{gather*}
\noindent Since rules for $aux$ are unlabelled, the value of $aux(s^n(z))$ in the least model is $I(aux(s^n(z))=1$ so the effect of this predicate is somehow ``transparent'' regarding causal justifications. As a result, the value of $nat(s^n(z))$ is just $l(s^n(z))$ as before.
\end{example}

\newpage
\section*{Appendix D. Proofs}\label{sec:proofs}

In order to improve clarity, for any causal graph $G=\tuple{V,E}$, vertices $v_1$ and $v_2$ and edge $(v_1,v_2)$ we respectively use the notation $v_1 \in G$ and $(v_1,v_2) \in G$ instead of $v_1 \in V$ and $(v_1,v_2) \in E$.

\subsection{Causes as graphs}

\begin{proposition}[Monotonicity]
\label{prop:graphs.monotonicity}
Let $G, G'$ be a pair of causal graph with $G \leq G'$. Then, for any causal graph $H$:

\centering{ $G*H \leq G'*H$, \hspace{0.5cm}
$G\cdot H \leq G' \cdot H $ \hspace{0.5cm}
and \hspace{0.5cm}$H \cdot G \leq H \cdot G'$}
\end{proposition}
\begin{proof}
First we will show that $G*H\leq G'*H$.
Suppose that $E(G*H)\not\supseteq E(G'*H)$
and let $(l_1,l_2)$ be an edge in $E(G'*H)$ but not in $E(G*H')$,
i.e. $(l_1,l_2)\in E(G'*H)\backslash E(G*H')$.

Thus, since by product definition $E(G'*H)=E(G')\cup E(H)$,
it follows that
either $(l_1,l_2)\in E(G')$ or $(l_1,l_2)\in E(H)$.
It is clear that if $(l_1,l_2)\in E(H)$ then
$(l_1,l_2)\in E(G*H)=E(G)\cup E(H)$.
Furthermore, since $G\leq G'$ it follows that $E(G)\supseteq E(G')$,
if $(l_1,l_2)\in E(G')$ then $(l_1,l_2)\in E(G)$ and consequently $(l_1,l_2)\in E(G*H)=E(G)\cup E(H)$.
That is $E(G*H)\supseteq E(G'*H)$
and then $G*H\leq G'*H$.
Note that $V(G*H)\supseteq V(G'*H)$ follows directly from $E(G*H)\supseteq E(G'*H)$ and the fact that every vertex has and edge to itself.

To show that $G\cdot H\leq G'\cdot H$ (the case for $H\cdot G\leq H\cdot G'$ is analogous) we have has to show that, in addition to the previous, for every edge $(l_G,l_H)\in E(G'\cdot H)$
with $l_G\in V(G')$ and $l_H\in V(H)$
it holds that $(l_G,l_H)\in E(G\cdot H)$.
Simply note that since $G\leq G'$ it follows $V(G)\supseteq V(G)'$ and then $l_G\in V(G)$. Consequently $(l_G,l_H)\in E(G\cdot H)$.\qed
\end{proof}

\begin{proposition}[Application associativity for causal graphs]\label{prop:graph.appl.associative}
Let $G_1$, $G_2$ and $G_3$ be three causal graphs. Then $G_1 \cdot (G_2 \cdot G_3)  \ = \ ( G_1 \cdot G_2 ) \cdot G_3$.
\end{proposition}
\begin{proof}
By definition, it follows that
\begin{align}
V\big( (G_1\cdot G_2) \cdot G_3 \big)
	& \ = \ 
	( V(G_1) \cup V(G_2) ) \cup V(G_3)
	\ = \ V(G_1) \cup (V(G_2) \cup V(G_3))
	\label{eq:prop:graph.appl.associative.3}
\\
E\big( (G_1\cdot G_2) \cdot G_3 \big)
	& \ = \ 
	\big( E(G_1) \cup E(G_2) \cup E_{12} \big) \cup E(G_3) \cup E_{12,3}
	\notag
	\\
	& \ = E(G_1) \cup E(G_2) \cup E(G_3) \cup E_{12} \cup E_{12,3}
\\
E\big( G_1\cdot (G_2 \cdot G_3) \big)
	& \ = \ 
	E(G_1) \cup \big( E(G_2) \cup E(G_3) \cup E_{23} \big) \cup E_{1,23}
	\notag
	\\
	& \ = E(G_1) \cup E(G_2) \cup E(G_3) \cup E_{1,23} \cup E_{23}
	\label{eq:prop:graph.appl.associative.4}
\end{align}
where
\begin{align*}
E_{12} &\ \eqdef \
	\setm{ (v_1,v_2) }{ v_1 \in V_1 \text{ and } v_2 \in V_2}
	&
E_{12,3} &\ \eqdef \
	\setm{ (v_{12},v_3) }
		{ v_{12} \in V_1 \cup V_2  \text{ and } v_3 \in V_3}
\\
E_{23} &\ \eqdef \
	\setm{ (v_2,v_3) }{ v_2 \in V_2 \text{ and } v_3 \in V_3}
	&
E_{1,23} &\ \eqdef \
	\setm{ (v_1,v_{23}) }
	{ v_1 \in V_1 \text{ and } v_{23} \in V_2 \cup V_3 }
\end{align*}
From \eqref{eq:prop:graph.appl.associative.3}~and~\eqref{eq:prop:graph.appl.associative.4}, it follows that
\begin{align*}
(G_1\cdot G_2) \cdot G_3 =  G_1\cdot (G_2 \cdot G_3)
	&&\text{ if and only if }
	& E_{12} \cup E_{12,3} = E_{1,23} \cup E_{23}
\\
	&&\text{ if and only if }
	&	E_{12,3} \subseteq E_{1,23} \cup E_{23}
		\text{ and }
		E_{1,23} \subseteq E_{12} \cup E_{12,3}
\end{align*}
Note that $E_{12}\subseteq E_{1,23}$ and $E_{23}\subseteq E_{12,3}$.
Then, we will show that $E_{12,3} \subseteq E_{1,23} \cup E_{23}$.
Suppose there is an edge $(v_{12},v_3)\in E_{12,3}$ such that
 $(v_{12},v_3)\notin E_{23}$ and  $(v_{12},v_3)\notin E_{1,23}$.
 Since $(v_{12},v_3)$, $v_{12}\in V_1\cup V_2$ and $v_3\in V_3$.
 If $v_{12}\in V_1$, then  $(v_{12},v_3)\in E_{12,3}$ which is a contradiction, and if, otherwise, $v_{12}\in V_2$, then $(v_{12},v_3)\in E_{23}$ which is also a contradiction. The case $E_{1,23} \subseteq E_{12} \cup E_{12,3}$ is symmetric.
\end{proof}

\begin{proposition}\label{prop:graph.NF}
For every causal graph $G=\tuple{V,E}$ it holds that $G = \prod\setbm{ l \cdot l' }{ (l,l') \in E }$.
\end{proposition}
\begin{proof}
Let $G'$ be a causal graphb s.t.
$G'=\prod\setbm{ l \cdot l' }{ (l,l') \in E }$.
Then for every edge $(l,l')\in E(G)$ it holds that
$(l,l')\in E(l\cdot l')$ and then
$(l,l')\in E(G')=\bigcup\setm{E(l\cdot l')}{(l,l')\in E}$,
i.e. $E(G)\subseteq E(G')$.
Furthermore for every $(l,l')\in E(G')$
there is $l_i\cdot l_j$ s.t. $(l,l')\in l_i\cdot l_j$ and $(l_i,l_j)\in E(G)$. Then, since $E(l_i\cdot l_j)=\set{(l_i,l_j)}$ it follows that $(l,l')\in E(G)$, i.e. $E(G)\supseteq E(G')$.
Consequently $G=G'=\prod\setbm{ l \cdot l' }{ (l,l') \in E }$.\qed
\end{proof}

\begin{proposition}[Infimum]\label{prop:prod.is.infimum}
Any set of causal graphs $S$ has a $\leq$-infimum given by their product $\prod S$.
\end{proposition}
\begin{proof}
By definition $\prod S$ is the causal graph whose vertices and edges are respectively the sets
$V(\prod S)=\bigcup\setbm{ V(G) }{ G\in S }$ and
$E(\prod S)=\bigcup\setbm{ E(G) }{ G\in S }$.
It is easy to see that $\prod S$ is the supremum of the subgraph relation, so that,
since for every pair of causal graphs $G\leq G'$ iff $G\supseteq G'$, it follows that infimum of $S$ w.r.t. $\leq$.\qed
\end{proof}

\begin{proposition}[Application distributivity w.r.t. products over causal graphs]
\label{prop:graphs.appl.distr.over.prods}
For every pair of sets of causal graphs $S$ and $S'$, it holds that
\begin{IEEEeqnarray*}{c+x*}
\big( \prod S \big) \cdot \big( \prod S' \big)
	= \prod \setbm{ G\cdot G' }{ G \in S \text{ and } G'\in S' }. \hspace{82pt} &
\end{IEEEeqnarray*}
\end{proposition}
\begin{proof}
For readability sake, we define two causal graphs
\begin{align*}
G_R &\eqdef \big( \prod S \big) \cdot \big( \prod S' \big)
&\text{and}&&
G_L &\eqdef \prod \setbm{ G\cdot G' }{ G \in S \text{ and } G'\in S' } 
\end{align*}
and we assume that both $S$ and $S'$ are not empty sets.
Note that $\prod\emptyset=\causes=\prod\set{G_\emptyset}$.
Then, by product definition, it follows that
\begin{IEEEeqnarray*}{rClCl C l}
E(G_L)	&=& \Big(
			&\bigcup&\setbm{ E(G) }{ G\in S}
			\cup 
			\bigcup\setbm{ E(G') }{ G\in S'}
			\cup E_L
			\Big)^*
\\
E(G_R)	&=& \Big(
			&\bigcup&
			\setbm{
				E(G)\cup E(G')\cup E_R(G,G')
			}{ G\in S \text{ and } G'\in S' }
			\Big)^*
\end{IEEEeqnarray*}
where
\vspace{-2\topsep}
\begin{IEEEeqnarray*}{cCl}
E_L	&=&\setbm{ (l,l') }
		{ l\in \bigcup\setm{V(G)}{G\in S}
		\text{ and }
		l'\in \bigcup\setm{V(G')}{G'\in S'}}
\\
E_R(G,G')
	&=&\setm{ (l,l') }
					{ l\in V(G) \text{ and } l'\in V(G') }
\end{IEEEeqnarray*}
Furthermore let
$E_R=\bigcup\setm{E_R(G,G')}{ G\in S \text{ and } G'\in S' }$.
For every edge $(l,l')\in E_L$ there are a pair of c-graphs $G\in S$ and $G'\in S'$ s.t. $l\in V(G)$ and $l'\in V(G')$ and then
$(l,l')\in E_R(G,G')$ and so $(l,l')\in E_R$.
Moreover, for every edge $(l,l')\in E_R$ there are a pair of c-graphs $G\in S$ and $G'\in S'$ s.t. $(l,l')\in E_R(G,G')$ with $l\in V(G)$ and $l'\in V(G)'$. So that $(l,l')\in E_L$.
That is $E_L=E_R$.
Then
\begin{IEEEeqnarray*}{rCCCl}
E(G_R)	&=\Big(
		&&\bigcup&
		\setbm{ E(G) }{ G\in S} \cup 
			\bigcup\setbm{ E(G') }{ G'\in S'}
		\cup E_R
		\Big)^*
\\		&=&\IEEEeqnarraymulticol{3}{l}{
			\big(E(G_L)\backslash E_L \cup E_R	\big)^*
		=\big(E(G_L)\big)^* = E(G_L)
		}
\end{IEEEeqnarray*}
Consequently $G_L=G_R$.\qed
\end{proof}

\begin{proposition}[Transitive application distributivity w.r.t. products over causal graphs]
\label{prop:graphs.appl.distr.over.prods.cont}
For any causal graphs $G$, $G'\neq\emptyset$ and $G''$, it holds that
\begin{align*}
	G \cdot G' \cdot G'' = G \cdot G' * G' \cdot G''
\end{align*}
\end{proposition}
\begin{proof}
It is clear that
$G \cdot G' \cdot G''\leq G \cdot G'$
and
$G \cdot G' \cdot G''\leq G' \cdot G''$
and then
$G \cdot G' \cdot G'' \leq G \cdot G' * G' \cdot G''$.
Let $G_{1}$, $G_{2}$, $G_{L}$ and $G_R$ be respectively
$G_1=G \cdot G'$,
$G_2=G' \cdot G''$,
$G_L=G_1 \cdot G''$,
and
$G_R=G_1*G_2$.
Suppose that
$G_L<G_R$,
i.e. $G_L\supset G_R$ and there is an edge
$(v_1,v_2)\in G_L$ but $(v_1,v_2)\notin G_R$.
Then $G_1\subseteq G_R$ and $G''\subseteq G_2\subseteq G_R$ and
one of the following conditions holds:
\vspace{-\topsep}
\begin{enumerate}
\item $(v_1,v_2)\in G_1\subseteq G_R$
or $(v_1,v_2)\in G''\subseteq G_R$ which is a contradiction with
$(v_1,v_2)\notin G_R$.
\item $v_1\in G_1$ and $v_2\in G''$,
i.e. $v_1\in G$ and $v_2\in G''$
or $v_1\in G'$ and $v_2\in G''$.
Furthermore, if the last it is clear that
$(v_1,v_2)\in G'\cdot G''=G_2\subseteq G_R$
which is a contradiction with
$(v_1,v_2)\notin G_R$.
\end{enumerate}
Thus it must be that
$v_1\in G$ and $v_2\in G''$.
But then, since $G'\neq\emptyset$
there is some $v'\in G'$
and consequently there are
edges
$(v_1,v')\in G\cdot G'=G_1\subseteq G_R$
and
$(v',v_2)\in G'\cdot G''=G_2\subseteq G_R$.
Since $G_R$ is closed transitively,
$(v_1,v_2)\in G_R$ which is a contradiction with the assumption that
\mbox{$(v_1,v_2)\notin G_R$}.
That is, $G_L=G \cdot G' \cdot G'' = G \cdot G' * G' \cdot G''=G_R$.
\end{proof}

\begin{proposition}[Application idempotence w.r.t. singleton causal graphs]
For any causal graph $G$ whose only edge is $(l,l)$, $G \cdot G = G$.
\end{proposition}
\begin{proof}
By definition $G \cdot G = G \cup G \cup \setm{ (v_1,v_2) }{ v_1 \in G \text{ and } v_2 \in G }$. Since the only vertex of $G$ is $l$, it follows that $G\cdot G = G \cup \set{ (l,l) } = G$.\qed
\end{proof}

\begin{proposition}[Application absorption over causal grapsh]
Let $G_1$, $G_2$ and $G_3$ be three causal graphs. Then $G_1 \cdot G_2 \cdot G_3 = G_2 * G_1 \cdot G_2 \cdot G_3$.
\end{proposition}
\begin{proof}
By definition, it is clear that $G_1\cdot G_2 \cdot G_3 \supseteq G_2$ and then
$$G_2 * G_1 \cdot G_2 \cdot G_3 = (G_2 \cup G_1 \cdot G_2 \cdot G_3)^* = (G_1 \cdot G_2 \cdot G_3)^* = G_1 \cdot G_2 \cdot G_3$$
\end{proof}

\begin{proposition}[Absorption Extended]\label{prop:term.absorption.deriv}
Let $a$ and $b$ be elements of an algebra holding the identity and absorption equivalences showed in Figure~\ref{fig:appl}. Then
\begin{align*}
a * a \cdot b &= a \cdot b &&&
a * b \cdot a &= b \cdot a
\end{align*}
\end{proposition}
\begin{proof}
The proof follow from the following equivalences:
\begin{gather*}
\begin{aligned}
a * a \cdot b
	&= a * 1 \cdot a \cdot b	&&\text{(identity)}\\
	&= 1 \cdot a \cdot b		&&\text{(absorption)}\\
	&= a \cdot b				&&\text{(identity)}
\end{aligned}
\hspace{1cm}
\begin{aligned}
a * b \cdot a
	&= a * b \cdot a \cdot 1	&&\text{(identity)}\\
	&= b \cdot a \cdot 1		&&\text{(absorption)}\\
	&= b \cdot a				&&\text{(identity)}
\end{aligned}
\end{gather*}
\end{proof}

\newpage

\begin{proposition}[Transitivity extended]\label{prop:term.nosum.transitivity}
Let $a$, $b$ and $c$ be elements of an algebra holding the identity and absorption equivalences showed in Figure~\ref{fig:appl} such that $b$ is different from $1$. Then
\begin{align*}
a \cdot b * b \cdot c
	= a \cdot b * b \cdot c * a \cdot c
\end{align*}
follows from application associative, identity and absorption and distributivity over products.
\end{proposition}
\begin{proof}
\begin{align*}
a \cdot b * b\cdot c * a \cdot c
	&= (a \cdot b * b \cdot c) * a \cdot c
	&&\text{(associative `$*$')}
\\
	&= (a \cdot b \cdot c) * a \cdot c
	&&\text{(transitivity)}
\\
	&= a \cdot (b \cdot c) * a \cdot c
	&&\text{(associative '$\cdot$')}
\\
	&= a \cdot ( b \cdot c *  c)
 	&&\text{(distributivity)}
\\
	&= a \cdot ( b \cdot c )
 	&&\text{(absorption ext.)}
\\
	&= a \cdot b \cdot c
	&&\text{(associativity)}
\\
	&= a \cdot b * b \cdot c
	&&\text{(transitivity)}
\end{align*}
\end{proof}

\begin{corollary}
Given causal graphs $G_1$, $G_2$, $G_3$ and $G_l$ such that $(l,l)$ is the only edge of $G_l$, the equivalences reflected in the Figures~\ref{fig:equivalences.causal.graphs}~and~\ref{fig:equivalences.causal.graphs.der} hold.
\end{corollary}

\begin{figure}[htbp]
\footnotesize
\begin{center}
\newcommand{\titleSep}{-5pt}
\newcommand{\contentSep}{-13pt}
\newcommand{\rowSep}{5pt}
$
\begin{array}{c}
\hbox{\em Idempotence}\vspace{\titleSep}\\
\hline\vspace{\contentSep}\\
\begin{array}{r@{\ }c@{\ }l c l}
G_l & \cdot & G_l & = & G_l\\
\\
\end{array}
\end{array}
$
\ \ \ \
$
\begin{array}{c}
\hbox{\em Associativity}\vspace{\titleSep}\\
\hline\vspace{\contentSep}\\
\begin{array}{r@{\ }c@{\ }r@{}c@{}l c r@{}c@{}l@{\ }c@{\ }l@{\ }}
G_1 & \cdot & (G_2 & \cdot & G_3) & = & (G_1 & \cdot & G_2) & \cdot & G_3\\
\\
\end{array}
\end{array}
$
\ \ \ \
$
\begin{array}{c}
\hbox{\em Product\ distributivity}\vspace{\titleSep}\\
\hline\vspace{\contentSep}\\
\begin{array}{rcl}
G_1 \cdot (G_2*G_3)     & = & (G_1 \cdot G_2) * (G_1 \cdot G_3) \\
(G_1*G_2) \cdot G_3     & = & (G_1 \cdot G_3) * (G_2 \cdot G_3)
\end{array}
\end{array}
$
\ \ \ \
$
\begin{array}{c}
\hbox{\em Absorption}\vspace{\titleSep}\\
\hline\vspace{\contentSep}\\
\begin{array}{r@{\ }c@{\ }r@{}c@{}c@{}c@{}l c r@{}c@{}c@{}c@{}l }
G_2 & * & G_1 & \cdot & G_2 & \cdot & G_3 & = & G_1 & \cdot & G_2 & \cdot & G_3\\
\\
\end{array}
\end{array}
$
\ \ \ \
$
\begin{array}{c}
\hbox{\em Identity }\vspace{\titleSep}\\
\hline\vspace{\contentSep}\\
\begin{array}{r@{\ }c@{\ }l@{\ \ } c@{\ \ } r@{}c@{}l }
G_1 & \cdot & G_\emptyset & = & G_1\\
G_\emptyset & \cdot & G_1 & = & G_1\\
\end{array}
\end{array}
$
\ \ \ \
$
\begin{array}{c}
\hbox{\em Transitivity}\vspace{\titleSep}\\
\hline\vspace{\contentSep}\\
\begin{array}{r@{\ }c@{\ }l }
G_1 \! \cdot \! G_2 \! \cdot \! G_3 
	& = & (G_1 \! \cdot \! G_2) * (G_2 \! \cdot \! G_3) \ \ \hbox{with} \ G_2 \neq G_\emptyset \\
\\
\end{array}
\end{array}
$
\end{center}
\vspace{-5pt}
\caption{Properties of the `$\cdot$' and `$*$' operators over causal graphs ($G_l$ only contains the edge $(l,l)$).}
\label{fig:equivalences.causal.graphs}
\end{figure}

\begin{figure}[htbp]
\footnotesize
\begin{center}
\newcommand{\titleSep}{-5pt}
\newcommand{\contentSep}{-13pt}
\newcommand{\rowSep}{5pt}
$
\begin{array}{c}
\hbox{\em Absorption (der)}\vspace{\titleSep}\\
\hline\vspace{\contentSep}\\
\begin{array}{r@{\ }c@{\ }l c r@{}c@{}l }
G_1 & * & ( G_1 \cdot G_2 ) & = & G_1 & \cdot & G_2\\
G_1 & * & ( G_2 \cdot G_1 ) & = & G_2 & \cdot & G_1
\end{array}
\end{array}
$
\ \ \ \
$
\begin{array}{c}
\hbox{\em Transitivity (der)}\vspace{\titleSep}\\
\hline\vspace{\contentSep}\\
\begin{array}{r@{\ }c@{\ }l}
( G_1 \! \cdot \! G_2 ) * ( G_2 \! \cdot \! G_3 )
	& = &
	(G_1 \! \cdot \! G_2) * (G_2 \! \cdot \! G_3) * (G_1 \! \cdot \! G_3)
\\
\\
\end{array}
\end{array}
$
\end{center}
\vspace{-5pt}
\caption{Properties following from those in Figure~\ref{fig:equivalences.causal.graphs}.}
\label{fig:equivalences.causal.graphs.der}
\end{figure}

\newpage

\begin{theorem}\label{thm:graphs.free.algebra}
Given a set of labels $Lb$, $\tuple{\causes,*,\cdot}$ is the free algebra generated by $Lb$ defined by equations in the Figure~\ref{fig:equivalences.causal.graphs}, i.e. the mapping $graph:Lb\longrightarrow\causes$ mapping each label $l$ to the graph $G_l$ containing the only edge $(l,l)$ is an injective (preserving-idempotence) homomorphism and for any set $F$ and idempotence-preserving map $\delta:Lb\longrightarrow F$, there exists a homomorphism \mbox{$term:\causes\longrightarrow F$} defined as \mbox{$term(G) \mapsto \prod\setm{ \delta(l_1) \cdot \delta(l_2) }{ (l_1,l_2) \in G }$} such that $\delta = term \circ graph$.
\end{theorem}
\begin{proof}
For clarity sake we omit the idempotence-preserving mapping $\delta$ and we write just $l$ instead of $\delta(l)$. Thus the mapping $term:\causes\longrightarrow F$ is just $term(G) \mapsto \prod\setm{ l_1 \cdotl l_2 }{ (l_1,l_2) \in G }$. We start showing that $graph$ preservers the `$\cdot$' idempotence equation.
\begin{align*}
graph(l) \cdot graph(l)
	&= G_l \cdot G_l
	 = G_l \cup G_l
	 \cup \setm{ (l_1,l_2) }{ l_1 \in G_l \text{ and } l_2 \in G_l}
\\
	&= G_l \cup G_l \cup \set{ (l,l) } = G_l \cup G_l \cup G_l = G_l = graph(l)
\end{align*}
Furthermore, suppose $graph(l_1) = graph(l_2)$, i.e. $\set{ (l_1,l_2) }=\set{ (l_2,l_2) }$. Then $l_1 = l_2$. Hence $graph$ is an injective homomorphism. Now we show that $term$ preserves the `$*$'. Let $D = \bigcup_{G \in U} G$. Then
\begin{align*}
\prod_{G \in U} term(G)
	&= \prod_{G \in U} \prod_{ (l_1,l_2) \in G }  l_1 \cdotl l_2 
	 = \prod_{ (l_1,l_2) \in D }  l_1 \cdotl l_2
\\
	&= \prod_{ (l_1,l_2) \in D^* }  l_1 \cdotl l_2 = term(D^*) 
\\
	&= term\Big( \Big( \bigcup_{G \in U} G \Big)^* \Big)
	= term\Big( \prod_{G \in U} G \Big)
\end{align*}
We will show now that $term$ also preserves `$\cdot$'.
\begin{align*}
term(G_1) \cdot term(G_2)
	&= \Big( \prod_{ (u_1,u_2) \in G_u }  u_1 \cdotl u_2 \Big)
		\cdot
		\Big( \prod_{ (v_1,v_2) \in G_v }  v_1 \cdotl v_2 \Big)
\\
	&= \prod_{ (u_1,u_2) \in G_u, \ (v_1,v_2) \in G_v } 
			(u_1\cdotl u_2) \cdot (v_1\cdotl v_2)
\\
		&= \prod_{ (u_1,u_2) \in G_u, \ (v_1,v_2) \in G_v } 
			u_1\cdot ( u_2 \cdotl v_1\cdotl v_2)
\\
		&= \prod_{ (u_1,u_2) \in G_u, \ (v_1,v_2) \in G_v } 
			u_1\cdot ( u_2 \cdotl v_1 * v_1 \cdotl v_2)
\\
		&= \prod_{ (u_1,u_2) \in G_u, \ (v_1,v_2) \in G_v } 
			u_1\cdotl u_2 \cdotl v_1 * u_1\cdotl v_1 \cdotl v_2
\\
		&= \prod_{ (u_1,u_2) \in G_u, \ (v_1,v_2) \in G_v } 
			u_1\cdotl u_2 * u_2 \cdotl v_1 * u_1\cdotl v_1 * v_1 \cdotl v_2 * u_1 \cdotl v_1 * u_1 \cdotl v_2
\\
		&= \prod_{ (u_1,u_2) \in G_u, \ (v_1,v_2) \in G_v } 
				(u_1\cdotl u_2 * v_1 \cdotl v_2) *
				(u_2 \cdotl v_1 * u_1\cdotl v_1 * u_1 \cdotl v_1 * u_1 \cdotl v_2)
\\
		&= \prod_{ (u_1,u_2) \in G_u, \ (v_1,v_2) \in G_v } 
				(u_1\cdotl u_2 * v_1 \cdotl v_2)
			*
			\prod_{ (u_1,u_2) \in G_u, \ (v_1,v_2) \in G_v }
				(u_2 \cdotl v_1 * u_1\cdotl v_1 * u_1 \cdotl v_1 * u_1 \cdotl v_2)
\\
		&= \prod_{ (l_1,l_2) \in G_u \cup G_v } 
				(l_1\cdotl l_2 )
			*
			\prod_{ u \in G_u, \ v \in G_v } (u \cdotl v)		
\\
		&= \prod_{ (l_1,l_2) \in G_u \cdot G_v } 
				(l_1\cdotl l_2 )
		 = term(G_1 \cdot G_2)
\end{align*}
Finally note that $term(graph(l))=l\cdotl l = l$ (note that $l$ stands for the more formally~$\delta(l)$).
\end{proof}

\newpage

\subsection{Alternative causes}

We define an explanation formed by a set of causal graphs $S$ as:
\begin{align*}
\sum S &= G 	&&\text{if } G \in S \text{ and } G*G' = G' \text{ for all } G' \in S
\end{align*}

\begin{theorem}\label{theorem:values.free.algebra.from.graphs}
The set of causal values $\VLb$ with the opperations join `$+$' and meet `$*$ forms the free, complete distributive lattice with generated by the set of causal graphs $\causes$ and further $\downarrow:\causes\longrightarrow\VLb$ is an injective homomorphism from the algebra $\tuple{\causes,+,*,\cdot}$ to $\tuple{\VLb,+,*,\cdot}$.
.\qed
\end{theorem}
\begin{proofof}{Theorem~\ref{theorem:algebra}}
Let $\Fi$ be the set of filters over the lower semilattice $\tuple{\causes,*}$.
Stumme was showed in \cite{Stumme97} that the concept lattice
$\cB\tuple{\Fi,\VLb,\Delta}$
(with $F \Delta I \Leftrightarrow F \cap I \neq \emptyset$) is isomorphic to the free completely distributive complete lattice generated by the partial lattice
$\tuple{\causes,+,*}$ where $+$ and $*$ are two partial functions corresponding with the supremum and infimum.
In our particular, case for every set of causal graphs $S$ its infimum is defined as $\prod S$
and the supremum is defined as $G\in S$
such that $G'\leq G$ for all $G'\in S$,
when such $G$ exists and undefined otherwise.
Thus $\VLb$ is the set of ideals over the partial lattice $\tuple{\causes,+,*}$, i.e. every $I\in\VLb$ is also  closed under defined suprema.
He also show that the elements of such lattice are described as pairs 
\begin{align*}
\setbm{ (\Fi_t,\Id_t) }
	{	\Fi_t\subseteq\Fi,\
		\Id_t\subseteq\VLb,\
		\Fi_t^I=\Id_t
		\text{ and }
		\Fi_t=\Id_t^I
	}
\end{align*}
where\vspace{-0.5cm}
\begin{align*}
\Fi_t^I&=\setbm{ I\in\Id }{ \forall F\in\Fi_t: F\cap I\neq\emptyset}
\\
\Id_t^I&=\setbm{ F\in\Fi }{ \forall I\in\Id_t: F\cap I\neq\emptyset}
\end{align*}
That is, every element is in the form $\tuple{\Id_t^I,\Id_t}$.
Furthermore infima and suprema can be described as follows:
\begin{IEEEeqnarray*}{rCrclCrcl}
\bigwedge_{t\in T}(\Id_t^I,\Id_t)
	&=& \Big(		&\bigcap_{t\in T} \Id_t^I&	,
			\Big(	&\bigcup_{t\in T} \Id_t \Big)^{II} \Big)
\\
\bigvee_{t\in T}(\Id_t^I,\Id_t)
	&=& \Big( 
			\Big(	&\bigcup_{t\in T} \Id_t^I& \Big)^{II},
					&\bigcap_{t\in T} \Id_t&  \Big)
\end{IEEEeqnarray*}
We will show that
$\epsilon_I:\cB\longrightarrow\VLb$ given by
$(\Id_t^I,\Id_t)\mapsto\bigcap \Id_t$
is an isomorphism between $\tuple{\cB,\vee,\wedge}$ and
$\tuple{\VLb,\cup,\cap}$.
Note that,
since $\prod\emptyset=G_\emptyset$
it holds that the empty set is not close under defined infimum and then it is not a filter,
i.e. $\emptyset\not\in\Fi$,
and then
for every filter $F\in\Fi$
it holds that $G_\emptyset\in F$.
Thus if $\Id_t=\emptyset$ follows that
$\Id_t^I=\Fi$ and then
$\Id_t^{II}=\set{I\in\VLb\mid G_\emptyset\in I }=\causes\neq\Id_t$.
That is, $\tuple{\emptyset^I,\emptyset}\not\in\cB$.

We will show that for every ideal $I_t\in\Id$ and for
every set of ideals $\Id_t\subseteq\Id$
s.t. $I_t=\bigcap\Id_t$
it holds that
\begin{align}
(\Id_t^I,\Id_t)\in\cB\tuple{\Fi,\Id,\Delta}
	\Longleftrightarrow
	\Id_t=\setbm{ I\in\Id }{ I_t\subseteq I }
\end{align}
and consequently $\epsilon_I$ is a bijection between and  $\cB$ and $\VLb$.

Suppose that $(\Id_t^I,\Id_t)\in\cB\tuple{\Fi,\Id,\Delta}$.
For every $I\in\Id_t$ it holds that $I_t\subseteq I$.
So suppose there is $I\in\Id$ s.t. $I_t\subseteq I$ and $I\not\in\Id_t$.
Then there is $F\in\Id_t^I$ s.t. $I\cap F=\emptyset$
and for every element $I'\in\Id_t$ it holds that
$I'\cap F\neq\emptyset$.
Pick a causal graph $G$ s.t.
$G=\prod\set{G'\mid G'\in I'\cap F$ and $I'\in\Id_t}$.
Since for every $G'$ it holds $G'\in F$ and $G\leq G'$
follows that $G\in F$ ($F$ is close under infimum)
and $G\in I'$ (every $I'$ is close under $\leq$).
That is, for every $I'\in\Id_t$ it holds that
$G\in I'\cap F$ and then, since $I_t=\bigcap\Id_t$,
it also holds that $G\in I_t\cap F$
and since $I_t\subseteq I$
also $G\in I\cap F$
which contradict that $I\cap F=\emptyset$.
So that $I\in\Id_t$ and it holds that
\begin{align*}
(\Id_t^I,\Id_t)\in\cB\tuple{\Fi,\Id,\Delta}
	\Longrightarrow
	\Id_t=\setbm{ I\in\Id }{ I_t\subseteq I }
\end{align*}
Suppose that $\Id_t=\set{ I\in\Id\mid I_t\subseteq I }$
but
$(\Id_t^I,\Id_t)\in\cB\tuple{\Fi,\Id,\Delta}$,
i.e. $\Id_t\neq\Id_t^{II}$.
Note that $\Id_t\subseteq\Id_t^{II}$ because otherwise there are $I\in\Id_t$ and $F\in\Id_t^I$ s.t. $I\cap F=\emptyset$ which is a contradiction with the fact that for every $F\in\Id_t^I$ and $I\in\Id_t$ it holds that $F\cap I\neq\emptyset$.

So, there is $I\in\Id_t^{II}$ s.t. $I\not\in\Id_t$,
i.e. for every $F\in\Id_t^I$ it holds that
$F\cap I\neq\emptyset$
but $I_t\not\subseteq I$.
Pick $G\in I_t\backslash I$
and $F=\set{ G'\mid G\leq G'}$.
It is clear that $F\in\Fi$ and $F\cap I_t\neq\emptyset$
because $G\in I_t$,
so that $F\in\Id_t^I$.
Furthermore $F\cap I=\emptyset$, because $G\not\in I$,
which is a contradiction with the assumption. Thus
\begin{align*}
(\Id_t^I,\Id_t)\in\cB\tuple{\Fi,\Id,\Delta}
	\Longleftarrow
	\Id_t=\setbm{ I\in\Id }{ I_t\subseteq I }
\end{align*}

Now, we will show that
$(\Id_1^I,\Id_1)\vee(\Id_2^I,\Id_2)=(\Id_3^I,\Id_3)$
iff
$I_1\cup I_2=I_3$.
From the above statement follows that
\begin{align*}
\Id_1\cap\Id_2
	&=\set{I\in\Id\mid I_1\subseteq I \text{ and } I_2\subseteq I}=
\\
	&=\set{I\in\Id\mid I_1\cup I_2\subseteq I}
\\
\Id_3
	&=\set{I\in\Id\mid I_3 \subseteq I}
\end{align*}
That is,
$\Id_1\cap\Id_2=\Id_3$ iff $I_1\cup I_2=I_3$ and by definition of $\vee$ the first is equivalent to $(\Id_1^I,\Id_1)\vee(\Id_2^I,\Id_2)=(\Id_3^I,\Id_3)$.

Finally we will show that
$(\Id_1^I,\Id_1)\wedge(\Id_2^I,\Id_2)=(\Id^I_3,\Id_3)$ iff
$I_1\cap I_2=I_3$. It holds that
\begin{align*}
(\Id_1\cup\Id_2)^{II}
	&=\Big(
		\setbm{I\in\Id}{ I_1\subseteq I \text{ or } I_2\subseteq I}
		\Big)^{II}=
\\
&=\Big(\setbm{I\in\Id}{ I_1\cap I_2\subseteq I} \Big)^{II}
\\
\Id_3
	&=\setbm{I\in\Id}{ I_3 \subseteq I}
\end{align*}
Since $\epsilon_I$ is a bijection,
it holds that $(\Id_1\cup\Id_2)^{II}=\Id_3$ iff $I_1\cap I_2=I_3$.

Thus $\epsilon_I:\cB\longrightarrow\VLb$ is an isomorphism between 
$\tuple{\cB,\vee,\wedge}$ and
$\tuple{\VLb,\cup,\cap}$, i.e. $\tuple{\VLb,\cup,\cap}$ is isomorphic to the free completely distributive lattice generated by $\tuple{\causes,+,*}$.

Let's check now that $\downarrow:\causes\longrightarrow\VLb$ is an injective homomorphism.
Stumme has already showed that $\epsilon_p:\causes\longrightarrow\cB$ given by
\begin{gather*}
\epsilon_p(G)\mapsto
	\Big(	\setbm{F\in\Fi}{ G\in F},
			\setbm{I\in\Id}{ G\in I}
		\Big)
\end{gather*}
is an injective homomorphism between the partial lattice
$\tuple{\causes,+,*}$ and
$\tuple{\cB,\vee,\wedge}$.
So that $\epsilon_I\circ\epsilon_p$ is an injective homomorphism
between
$\tuple{\causes,+,*}$ and
$\tuple{\VLb,\cup,\cap}$ given by
\begin{gather*}
\epsilon_I\circ\epsilon_p(G)\mapsto
		\bigcap\setbm{I\in\VLb}{ G\in I} \ = \ \down{G}
\end{gather*}

Note that for any causal graph $G$ 
and $G'\in\causes$ s.t. $G'\leq G$
it holds that
\mbox{$G'\in \epsilon_I\circ\epsilon_p(G)$},
that is $\downarrow G\subseteq\epsilon_I\circ\epsilon_p(G)$.
Furthermore for every causal graph $G$
it holds that $\epsilon(G)$ is an ideal,
i.e. $\downarrow G\in\VLb$ and it is clear that $G\in \ \down{G}$
so that, $\epsilon_I\circ\epsilon_p$ is an intersection with $\downarrow G$ as one of its operands, thus $\epsilon_I\circ\epsilon_p(G)\subseteq\downarrow G$.
That is $\downarrow G=\epsilon_I\circ\epsilon_p(G)$
and consequently it is an injective homomorphism
between $\tuple{\causes,+,*}$
$\tuple{\VLb,\cup,\cap}$.

Let us show now that the mapping $\downarrow$ also preserves the `$\cdot$' operation. Take any causal graphs $G_1$ and $G_2$, then
\begin{align*}
\down{G_1}\ \cdot\down{G_2}
	& \ = \ \down{\setm{(G_1'\cdot G_2')}
		{ G_1' \in \ \down{G_1} \text{ and } G_2'\in \ \down{G_2} } }
\\
	& \ = \ \down{\setm{(G_1'\cdot G_2')}
		{ G_1' \leq G_1 \text{ and } G_2' \leq G_2 } }
	  \ = \ \down{(G_1\cdot G_2)}
\end{align*}
Note that, from Proposition~\ref{prop:graphs.monotonicity}, it follows that $G_1'\cdot G_2'\leq G_1\cdot G_2$. That is, $\downarrow$ in an homomorphism between $\tuple{\causes,+,*\cdot}$ and $\tuple{\VLb,+,*\cdot}$.
\end{proofof}

\begin{corollary}
Every equivalence showed in Figure~\ref{fig:DBLattice} hold. Futheremore equivalences showed in Figures~\ref{fig:equivalences.causal.graphs}~and~~\ref{fig:equivalences.causal.graphs.der} also hold if principal ideals $\down{G_1}$, $\down{G_2}$ and $\down{G_3}$ are considered instead of causal graphs $G_1$, $G_2$ and $G_3$.
\end{corollary}

\begin{corollary}
Given causal terms without sums $c$, $d$, $e$ and a label $l$, the equivalences reflected in the Figures~\ref{fig:appl.causal.term.nosum}~and~\ref{fig:appl.causal.term.nosum.der} hold.
\end{corollary}

\begin{figure}[htbp]
\begin{center}
\newcommand{\titleSep}{-5pt}
\newcommand{\contentSep}{-15pt}
\newcommand{\rowSep}{5pt}
$
\begin{array}{c}
\hbox{\em Idempotence}\vspace{\titleSep}\\
\hline\vspace{\contentSep}\\
\begin{array}{r@{\ }c@{\ }l c l}
l & \cdot & l & = & l\\
\\
\end{array}
\end{array}
$
\ \ \ \
$
\begin{array}{c}
\hbox{\em Associativity}\vspace{\titleSep}\\
\hline\vspace{\contentSep}\\
\begin{array}{r@{\ }c@{\ }r@{}c@{}l c r@{}c@{}l@{\ }c@{\ }l@{\ }}
c & \cdot & (d & \cdot & e) & = & (c & \cdot & d) & \cdot & e\\
\\
\end{array}
\end{array}
$
\ \ \ \
$
\begin{array}{c}
\hbox{\em Product\ distributivity}\vspace{\titleSep}\\
\hline\vspace{\contentSep}\\
\begin{array}{rcl}
 c \cdot ( d * e )
	& = &
	(c \cdot d) * (c \cdot e)
\\
(c * d) \cdot e
	& = &
	(c \cdot e) * (d \cdot e)
\end{array}
\end{array}
$
\ \ \ \
$
\begin{array}{c}
\hbox{\em Identity}\vspace{\titleSep}\\
\hline\vspace{\contentSep}\\
\begin{array}{r@{\ }c@{\ }l c l}
c & \cdot & 1 & = & c\\
1 & \cdot & c & = & c
\end{array}
\end{array}
$
\ \ \ \
$
\begin{array}{c}
\hbox{\em Absorption}\vspace{\titleSep}\\
\hline\vspace{\contentSep}\\
\begin{array}{r@{\ }c@{\ }r@{}c@{}c@{}c@{}l c r@{}c@{}c@{}c@{}l }
d & * & c & \cdot & d & \cdot & e & = & c & \cdot & d & \cdot & e\\
\\
\end{array}
\end{array}
$
\ \ \ \
$
\begin{array}{c}
\hbox{\em Transitivity}\vspace{\titleSep}\\
\hline\vspace{\contentSep}\\
\begin{array}{rcl}
c \cdot d \ \cdot e	& = & (c \cdot d) * (d \cdot e)
	\ \ \ \hbox{with} \ d \neq 1 \\
\\
\end{array}
\end{array}
$
\end{center}
\vspace{-5pt}
\caption{Properties of the `$\cdot$' and `$*$' operators over causal terms without `$+$'.}
\label{fig:appl.causal.term.nosum}
\end{figure}

\begin{figure}[htbp]
\begin{center}
\newcommand{\titleSep}{-5pt}
\newcommand{\contentSep}{-15pt}
\newcommand{\rowSep}{5pt}
$
\begin{array}{c}
\hbox{\em Absorption (der)}\vspace{\titleSep}\\
\hline\vspace{\contentSep}\\
\begin{array}{r@{\ }c@{\ }r@{}c@{}c@{}c@{}l c r@{}c@{}c@{}c@{}l }
d & * & c & \cdot & d &  & 		  & = & c & \cdot & d \\
d & * & d & \cdot & e &  &		  & = & d & \cdot & e
\end{array}
\end{array}
$
\ \ \ \
$
\begin{array}{c}
\hbox{\em Transitivity (der)}\vspace{\titleSep}\\
\hline\vspace{\contentSep}\\
\begin{array}{rcl}
c \cdot d * d \cdot e & = & c \cdot d * d \cdot e * c \cdot e\\
\\
\end{array}
\end{array}
$
\end{center}
\vspace{-5pt}
\caption{Additionally properties of the `$\cdot$' and `$*$' operators over causal terms without `$+$'.}
\label{fig:appl.causal.term.nosum.der}
\end{figure}

\newpage

\begin{proposition}[Application associativity]\label{prop:appl.associative}
Let $T$, $U$ and $W$ be three causal values. Then $T \cdot (U \cdot W)  \ = \ ( T \cdot U ) \cdot W$.
\end{proposition}
\begin{proof}
By definition, it follows that
\begin{align*}
( T \cdot U ) \cdot W
	& \ = \ \down{ \setm{ G_t \cdot G_u }{ G_t \in T \text{ and } G_u \in U } } \cdot W
\\
	& \ = \ \down{ \setm{ G' \cdot G_w }{ G' \leq G_t \cdot G_u,\ G_t \in T,\ G_u \in U \text{ and } G_w \in W } }
\end{align*}
It is clear that $G_t\cdot G_u \leq G_t \cdot G_u$ and then
\begin{align*}
\down{ \setm{ (G_t \cdot G_u) \cdot G_w }{ G_t \in T,\ G_u \in U \text{ and } G_w \in W } }
	& \ \subseteq \ \down{ \setm{ G' \cdot G_w }{ G' \leq G_t \cdot G_u,\ G_t \in T,\dotsc } }
\end{align*}
Furthermore since `$\cdot$' is monotonic, for every $G'\leq G_t \cdot G_u$, it holds that $G' \cdot G_w \leq (G_t \cdot G_u) \cdot G_w$ and then
\begin{align*}
( T \cdot U ) \cdot W
	& \ = \ 
\down{ \setm{ (G_t \cdot G_u) \cdot G_w }{ G_t \in T,\ G_u \in U \text{ and } G_w \in W } }
\end{align*}
Applying the same reasoning we can also conclude that
\begin{align*}
T \cdot ( U \cdot W )
	& \ = \ 
\down{ \setm{ G_t \cdot (G_u \cdot G_w ) }{ G_t \in T,\ G_u \in U \text{ and } G_w \in W } }
\end{align*}
That is, $T \cdot (U \cdot W)  \ = \ ( T \cdot U ) \cdot W$ holds whether, for every causal graph $G_t$, $G_u$ and $G_w$, it holds that $(G_t \cdot G_u) \cdot G_w \ = \ G_t \cdot (G_u \cdot G_w)$. This holds due to Proposition~\ref{prop:graph.appl.associative}.
\end{proof}

\begin{proposition}[Application distributivity w.r.t. additions]
Let $T$, $U$ and $W$ be three causal values. Then, it holds that $U \cdot (T + W)  \ = \ ( U \cdot T ) + (U \cdot W)$ and 
\mbox{$( U + T ) \cdot W \ = \ ( U \cdot W ) + ( T \cdot W)$}.
\end{proposition}
\begin{proof}
By definition, it follows that
\begin{align*}
(U \cdot T) + (U\cdot W)
	&= (U \cdot T) \cup (U\cdot W)
\\
	&= \ \down{ \setbm{ G_U \cdot G_T }{ G_U \in U \text{ and } G_T\in T} }
	\ \cup \ \down{ \setbm{ G_U \cdot G_T }{ G_U \in U \text{ and } G_W\in W} }
\\
	&= \ \down{ \setbm{ G_U \cdot G' }{ G_U \in U \text{ and } G'\in T\cup W} }
	= U \cdot ( T \cup W ) = U \cdot ( T + W )
\end{align*}
Furthermore \mbox{$( U + T ) \cdot W \ = \ ( U \cdot W ) + ( T \cdot W)$} holds symmetrically.
\end{proof}

\begin{proposition}[Application absorption]
Let $T$, $U$ and $W$ be three causal values. Then $T = T + \ U \cdot T \cdot W$ and $U \cdot T \cdot W = T * U \cdot T \cdot W$
\end{proposition}
\begin{proof}
From Proposition~\ref{prop:appl.associative} it follows that
\begin{align*}
U \cdot T \cdot W
	&= \ \down{\setbm{ G_U \cdot G_T \cdot G_W}{G_U \in U,\ G_T \in T \text{ and } G_W \in W}}
\end{align*}
Furthermore, for every c-graph $G_T\in T$, it holds that $G_U\cdot G_T \cdot T_W \leq G_T$. Then, since $T$ is an ideal, it follows that $G_U\cdot G_T \cdot T_W \in T$ and consequently $U\cdot T \cdot W \subseteq T$. Thus \mbox{$U \cdot T \cdot W \cup T = T$} and \mbox{$U \cdot T \cdot W \cap T = U \cdot T \cdot W$} and, by definition, these equalities can be rewritten as \mbox{$U \cdot T \cdot W + T = T$} and \mbox{$U \cdot T \cdot W * T = U \cdot T \cdot W$}.
\end{proof}

\begin{proposition}[Application identity and annihilator]
Given a causal value $T$, it holds that $T = 1 \cdot T$, $T= T \cdot 1$, $0 = T \cdot 0$ and $0 = 0 \cdot T$.
\end{proposition}
\begin{proof}
Note that $1$ and $0$ respectively correspond to $\causes$ and $\emptyset$ and by definition it follows that
\begin{align*}
1 \cdot T &= \ \down{ \setbm{G \cdot G_T}{ G\in\causes \text{ and } G_T \in T} } = T
\\
0 \cdot T &= \ \down{ \setbm{G \cdot G_T}{ G\in\emptyset \text{ and } G_T \in T} } = \emptyset = 0
\end{align*}
The other cases are symmetric.
\end{proof}

\begin{corollary}
The equivalences reflected in the Figure~\ref{fig:appl} hold.
\end{corollary}

\begin{theorem}\label{thm:free.algebra}
Given a set of labels $Lb$, $\tuple{\VLb,+,*,\cdot}$ is the free algebra generated by $Lb$ defined by equations in the Figure~\ref{fig:DBLattice}~and~\ref{fig:appl}, i.e. the mapping $value:Lb\longrightarrow\VLb$ mapping each label $l$ to the causal value $\down{G_l}$ with $G_l$ being the causal graph containing the only edge $(l,l)$ is an injective (preserving-idempotence) homomorphism and for any set $F$ and idempotence-preserving map $\delta:Lb\longrightarrow F$, there exists a homomorphism \mbox{$term:\VLb\longrightarrow F$} defined as \mbox{$term(U) \mapsto \sum\setm{ term(G) }{ G \in U }$} such that $\delta = term \circ graph$.
\end{theorem}
\begin{proof}
Note, from Theorems~\ref{thm:graphs.free.algebra}~and~\ref{theorem:values.free.algebra.from.graphs}, that $graph:Lb\longrightarrow\causes$ and $\downarrow:\causes\longrightarrow\VLb$ are injective homomorphisms respectively from $Lb$ to $\causes$ and from $\causes$ to $\VLb$. Furthermore the composition of injective homomorphisms is injective too. So that $value(l) \mapsto \ \down{graph(l)}$ is an injective homomorphism between $Lb$ and $\VLb$. We will show now that $term:\VLb\longrightarrow F$ preserves the `$+$':
\begin{align*}	 
term\Big( \sum_{U \in \sU} U \Big)
	 = term\Big( \bigcup\sU \Big)
	 = \sum_{G \in \bigcup \sU} term(G)
	 = \sum_{U \in \sU} \sum_{G \in U} term(G)
	 = \sum_{U \in \sU} term(U)
\end{align*}
that $term$ also preserves `$*$':
\begin{IEEEeqnarray*}{lCcCcCcCl}
\prod_{ U \in \sU} term(U)
	&=& \prod_{ U \in \sU} \sum_{ G \in U } term(G)
	&=& \sum_{ \varphi \in \Phi} \prod_{ U \in \sU}
	 		term(\varphi(U))
	&=& \sum_{ \varphi \in \Phi} term\Big( 
				\prod_{ U \in \sU} \varphi(U) \Big)
	&=& \sum_{t \in T_1} t
\\
term\Big( \prod_{ U \in \sU} U \Big)
	&=& term\Big( \bigcap \sU \Big)
	&=& \sum_{G \in \bigcap \sU} term(G)
	&&
	&=& \sum_{t \in T_2} t
\end{IEEEeqnarray*}
where $\Phi \ \eqdef \ \setm{ \varphi }{ \varphi(U) \in U \text{ with } U \in \sU }$.
Note that $G\in\bigcap\sU$ implies that $G\in U$ for all $U\in\sU$ and consequently there is $\varphi\in\Phi$ s.t. $\varphi(U)=G$ for all $U\in\sU$. Thus $term\big(\prod_{U\in\sU}\varphi(U)\big)=term\big(\prod_{U\in\sU} G\big)=term(G)$. That is $T_2 \subseteq T_1$. Note also that $\prod_{U\in\sU} \varphi(U) = \big( \bigcup_{U \in \sU } U \big)^* \supseteq \varphi(U)$ for all $U\in \sU$ and consequently $\prod_{U \in \sU} U \in U$ for all $U \in \sU$ and so that $\prod_{U \in \sU} U \in \bigcap \sU$ and $T_2 \supseteq T_1$.
And that $term$ preserves `$\cdot$' too:
\begin{align*}
term(U) \cdot term(W)
	&=  \Big( \sum_{G_u \in U} term(G_u) \Big) \cdot
		\Big( \sum_{G_w \in W} term(G_w) \Big)
\\
	&= \sum_{G_u \in U,\ G_w \in W}  term(G_u) \cdot term(G_w)
\\
	&= \sum_{G_u \in U,\ G_w \in W}  term(G_u \cdot G_w)
\\
	&= term\Big( \sum_{G_u \in U,\ G_w \in W}  G_u \cdot G_w \Big)
\\
	&= term\Big( \sum_{G_u \in U}  G_u \cdot \sum_{G_w \in W} G_w \Big)
\\
	&= term(U \cdot W)
\end{align*}
Finally $term(value(l))=\sum_{G \in value(l)} term(G)= term(G_l)$ and, from Theorem~\ref{thm:graphs.free.algebra}, \mbox{$term(G_l) = l$}.
\end{proof}

\newpage

\subsection{Positive programs}

\begin{lemma}
\label{lem:tp.properties}
Let $P$ be a positive (and possible infinite) logic program over signature $\tuple{At,Lb}$.
Then,
($i$) the least fix point of $T_P$, $\mathit{lfp}(T_P)$ is the least model of $P$, and
($ii$)  $\mathit{lfp}(T_P)=\tp{\omega}$.\qed
\end{lemma}
\begin{proof}
Since the set of causal values forms a lattice
causal logic programs can be translated to
\emph{Generalized Annotated Logic Programming} (GAP).  GAP
is a general a framework for multivalued logic programming where the set of truth values must to form an upper semilattice and rules (\emph{annotated clauses}) have  the following form:
\begin{align}
H : \rho
	\leftarrow B_1 : \mu_1 \ \&\ \dotsc \ \&\ B_n : \mu_n
	\label{eq:gap.rule}
\end{align}
where $L_0, \dotsc, L_m$ are literals, $\rho$ is an \emph{annotation} (may be just a truth value, an \emph{annotation variable} or a \emph{complex annotation}) and $\mu_1, \dotsc,\mu_n$ are values or annotation variables. A complex annotation is the result to apply a total continuous function to a tuple of annotations.
Thus a positive program $P$ is encoded in a GAP program, GAP$(P)$ rewriting each rule $\R \in \Pi$ of the form
\begin{align}
t: H \leftarrow B_1 \wedge \dotsc \wedge B_n 
	\label{eq:pos.rule}
\end{align}
as a rule GAP$(\R)$ in the form $\eqref{eq:gap.rule}$ where
$\mu_1, \dotsc,\mu_n$ are annotation that capture the causal values of each body literal and $\rho$ is a complex annotation defined as
$\rho=(\mu_1*\dotsc*\mu_n)\cdot t$.

Thus we will show that a causal interpretation $I \models \Pi$ if and only if $I \models^r$ GAP$(P)$ where $\models^r$ refers to the GAP restricted semantics.

For any program $P$ and interpretation $I$, by definition,
$I\models P$ (resp. $I\models^r$ GAP$(P)$)
iff $I\models\R$ (resp. $I\models^r$ GAP$(\R)$)
for every rule $\R\in P$.
Thus it is enough to show that for every rule $\R$ it holds that $I\models\R$ iff $I\models^r$ GAP$(\R)$.

By definition, for any rule $\R$ of the form of \eqref{eq:pos.rule} and an interpretation $I$, $I\models\R$ if and only if
$\big( I(B_1) * \dotsc * I(B_n) \big) \cdot t\leq I(H)$
whereas
for any rule GAP$(\R)$ in the form of \eqref{eq:gap.rule}, 
$I \models^r GAP(\R)$ iff for all $\mu_i \leq I(B_i)$ implies that $\rho=(\mu_1*\dotsc*\mu_n)\cdot t \leq I(H)$.

For the  only if direction,
take $\mu_i = I(B_i)$,
then
$\rho=(\mu_1*\dotsc*\mu_n)\cdot t
	=(I(B_1)*\dotsc*I(B_n))\cdot t$
and then $\rho\leq I(H)$ implies
$\big( I(B_1) * \dotsc * I(B_n) \big) \cdot t\leq I(H)$,
i.e. $I \models^r GAP(\R)$ implies $I\models\R$.
For the if direction, take $\mu_i\leq I(B_i)$ then,
since product an applications are monotonic operations,
it follows that
$(\mu_1*\dotsc*\mu_n)\cdot t\leq(I(B_1)*\dotsc*I(B_n))\cdot t\leq I(H)$,  That is, $I \models \R$ also implies $I \models^r GAP(\R)$.
Consequently $I \models \R$ iff $I \models^r GAP(\R)$.

Thus, from Theorem~1 in \cite{KiferS92}, it follows that the operator $T_P$ is monotonic.

To show that the operator $T_P$ is also continuous we
need to show that for every causal program $P$ the translation GAP$(P)$ is an \emph{acceptable} program.
Indeed since in a program GAP$(P)$ all body atoms are v-annotated it is \emph{acceptable}. Thus from Theorem~3 in \cite{KiferS92}, it follows that $\tp{\omega}=lfp(T_P)$ and this is the least model of $P$.
\end{proof}

\begin{lemma}\label{lem:tp.composition}
Given a positive and completely labelled program $P$, for every atom $p$ and integer $k\geq 1$,
\begin{align*}
\tpp{k}{p}
	&=\sum_{\R\in\Psi }
		\sum_{f \in \R} 
			\prod\setbm{f\big(\tpp{k-1}{q}\big)}{q\in body(\R) }
				\cdot label(\R)
\end{align*}
where $\Psi$ is the set of rules
$\Psi=\setm{\R\in\Pi}{head(\R)=p}$
and
$\R$ is the set of choice functions
$\R=\setbm{ f}{ f(S) \in S }$.
\end{lemma}
\begin{proof}
By definition of $\tpp{k}{p}$ it follows that
\begin{align*}
\tpp{k}{p}&=\sum\setbm{
		\big( \tpp{k-1}{q_1}*\dotsc*\tpp{k-1}{q_1} \big)
			\cdot label(\R) }
		{\R\in P \text{ with } head(\R)=p}
\end{align*}
then,
applying distributive of application w.r.t. to the sum and
and rewriting the sum and the product aggregating properly, it follows that
\begin{align*}
\tpp{k}{p}&=\sum_{\R\in\Psi }
			\prod\setbm{\tpp{k-1}{q}}{q\in body(\R) }
			\cdot label(\R)
\end{align*}
Furthermore for any atom $q$ the causal value $\tpp{k-1}{q}$
can be expressed as the sum of all c-graphs in it and then
\begin{align*}
\tpp{k}{p}&=\sum_{\R\in\Psi }
			\prod\setbm{
					\sum_{f\in\R} f\big(\tpp{k-1}{q}\big)}
				{q\in body(\R) }
			\cdot abel(\R)
\end{align*}
and applying distributivity of products over sums it follows that
\begin{IEEEeqnarray*}{rCl+x*}
\tpp{k}{p}&=&\sum_{\R\in\Psi } \sum_{f\in\R}
			\prod\setbm{
				 f\big(\tpp{k-1}{q}\big)}
				{q\in body(\R) }
			\cdot l_\R
			&\qed
\end{IEEEeqnarray*}
\end{proof}

\begin{lemma}\label{lem:tp.causes.composition}
Given a positive and completely labelled program $P$ and a causal graph $G$, for every atom $p$ and integer $k\geq 1$,
it holds that 
$G\in\tpp{k}{p}$ iff
there is a rule $l:p\leftarrow q_1,\dotsc, q_m$
and causal graphs $G_{q_1}$, $\dotsc$, $G_{q_m}$ respectively in
$\tpp{k-1}{q_i}$ and
$G\leq\big( G_{q_1}*\dotsc*G_{q_m} \big) \cdot l$.
\end{lemma}
\begin{proof}
From Lemma~\ref{lem:tp.composition} it follows that $G\in\tpp{k}{p}$ iff
\begin{align*}
G\in value\Big(\ \sum_{\R\in\Psi } \sum_{f\in\R}
			\prod\setbm{
				 f\big(\tpp{k-1}{q}\big)}
				{q\in body(\R) }
			\cdot label(\R) \ \Big)
\end{align*}
\vspace{-0.25cm}
iff
\vspace{-0.25cm}
\begin{align*}
G\in\bigcup_{\R\in\Psi} \bigcup_{f\in\R}
			value\Big( \prod\setbm{
				 \downarrow f\big(\tpp{k-1}{q}\big)\Big) }
				{q\in body(\R) } \cdot label(R)
\end{align*}
iff there is $\R\in\Phi$, with $head(\R)=p$ and a choice function $f\in\Psi$ s.t.
\begin{align*}
G\in value\Big(\prod\setbm{f\big(\tpp{k-1}{q}\big)}
				{q\in body(\R) } \cdot label(\R)\Big)
\end{align*}
Let $R=l:p\leftarrow q_1,\dotsc, q_m$ and
$f\big(\tpp{k-1}{q_i}\big)=G_{q_i}$.
Then the above can be rewritten as
$G\leq\big( G_{q_1}*\dotsc*G_{q_m} \big) \cdot l$.\qed
\end{proof}

\begin{definition}
Given a causal graph $G=\tuple{V,E}$,
we define the \emph{restriction} of $G$
to a set of vertex $V'\subseteq V$
as the casual graph $G'=\tuple{V',E'}$
where $E'=\setm{ (l_1,l_2)\in E }{ l_1\in V' \text{ and } l_2\in V' }$,
and
we define the \emph{reachable restriction} of $G$ to a set of vertex
$V'\subseteq V$, in symbols $G^{V'}$, as the restriction of $G$ to the set of vertex $V''$ from where  some vertex $l\in V'$ is reachable
$V''=\setm{ l' \in V }{ (l',l)\in E^* \text{ for some } l\in V'}$.
When $V'={l}$ is a singleton we write $G^l$.
\end{definition}

\begin{lemma}\label{lem:tp.restriction}
Let $P$ be a positive, completely labelled program,
$p$ and $q$ be atoms, $G$ be a causal graph,
$\R$ be a causal rule s.t. $head(\R)=q$ and $label(\R)=l$ is a vertex in $G$ and $k\in\set{1,\dotsc,\omega}$ be an ordinal.
If $G\in\tpp{k}{p}$, then $G^l\in\tpp{k}{q}$.\qed
\end{lemma}
\begin{proof}
In case that $k=0$ the lemma statement holds vacuous. Otherwise assume as induction hypothesis that the lemma statement holds for $k-1$.
From Lemma~\ref{lem:tp.causes.composition},
since $G\in\tpp{k}{p}$, 
there is a rule
\mbox{$\R_p=(l_p: p\leftarrow p_1,\dotsc,p_m)$}
and c-graph $G_{p_1},\dotsc,G_{p_m}$ s.t. each
\mbox{$G_{p_i}\in\tpp{k-1}{p_i}$}
and 
\mbox{$G\leq (G_{p_1}* \dotsc * G_{p_m} ) \cdot l_p$}.

If $l=l_p$ then, since $P$ is uniquely labelled, $\R=\R_p$,
$G^l=G$ and by assumption $G\in \tpp{k}{p}$.
Otherwise $l\in G_{p_i}$ for some $G_{p_i}$
and in its turn $G_{p_i}\in\tpp{k-1}{p_i}$.
By induction hypothesis $G^l\in\tpp{k-1}{q}$
and since $\tpp{k-1}{q}\subseteq\tpp{k}{q}$ it follows that \mbox{$G^l\in\tpp{k}{q}$}.

In case that $k=\omega$, by definition
$\tpp{\omega}{p}=\sum_{i<\omega}\tpp{i}{p}$
and the same for atom $q$.
Thus, if $G\in\tpp{\omega}{p}$
there is some $i<\omega$ s.t.
$G\in\tpp{i}{p}$,
and as we already show,
$G^l\in\tpp{i}{q}$
and consequently
$G^l\in\tpp{\omega}{q}$.
\qed
\end{proof}

\begin{lemma}\label{lem:tp.causes.composition.max}
Let $P$ be a positive, completely labelled program, $p$ be an atom and $G$ be a causal graph and $k\geq 1$ be an integer.
If $G$ is maximal in $\tpp{k}{p}$ then
\begin{enumerate}
\item there is a causal rule $\R=(l:p\leftarrow p_1,\dotsc,p_m)$ and there are causal graphs
$G_{p_1},\dotsc,G_{p_m}$ s.t. each
$G_{p_i}\in \max\tpp{k-1}{p_i}$ and 
\mbox{$G_p=(G_{p_1}* \dotsc * G_{p_m} ) \cdot l$} and

\item $l$ is not a vertex of any $G_{p_i}$.\qed
\end{enumerate}
\end{lemma}
\begin{proof}
From Lemma~\ref{lem:tp.causes.composition}
it follows that $G\in\tpp{k}{p}$ iff
there is a rule $\R=(l:p\leftarrow q_1,\dotsc,q_m)$ and causal graphs
$G_{q_1}',\dotsc,G_{q_m}'$ 
s.t. each
$G_{p_i}\in \max\tpp{k-1}{p_i}$
and
$G=\big( G_{q_1}'*\dotsc*G_{q_m}' \big) \cdot l$.
Let $G_{q_1},\dotsc,G_{q_m}$ be causes such that each
$G_{q_i}\in\max\tpp{k-1}{q_i}$ and $G_{q_i}'\leq G_{q_i}$ and let $G'$ be the c-graph
$G'=\big( G_{q_1}*\dotsc*G_{q_m} \big) \cdot l$.
By product and application monotonicity it holds that $G\leq G'$ and, again from Lemma~\ref{lem:tp.causes.composition},
it follows that $G'\in\tpp{k}{p}$.
Thus, since $G$ is maximal, it must to be that $G=G'$ and consequently
\mbox{$G=\big( G_{q_1}*\dotsc*G_{q_m} \big) \cdot l$}
where each $G_{q_i}$ is maximal.

Suppose that $l$ is a vertex of $G_{p_i}$ for some $G_{p_i}$.
From Lemma~\ref{lem:tp.restriction}, if follows that
$G_{p_i}^l\in\tpp{k}{p}$.
Furthermore, since $G_{p_i}\supseteq G_{p_i}^l$, it follows that $G_{p_i}\leq G_{p_i}^l$ and, since $l$ is a label ($l\neq 1$), it follows that $G<G_{p_i}$ and so that $G<G_{p_i}^l$ which contradicts the assumption that
\mbox{$G\in\max\tpp{k}{p}$}.\qed
\end{proof}

\begin{definition}
Given a causal graph $G$ we define $height(G)$ as the length of the maximal simple (no repeated vertices) path in $G$.
\end{definition}

\begin{lemma}\label{lem:tp.height}
Let $P$ be a positive, completely labelled program,
$p$ be an atom, $k\in\set{1,\dotsc,\omega}$ be an ordinal and $G$ be a causal graph.
If
\mbox{$G\in\max\tpp{k}{p}$}
and
$height(G)=h\leq k$
then
$G\in\tpp{h}{p}$.
\end{lemma}
\begin{proof}
In case that $h=0$, from Lemma~\ref{lem:tp.causes.composition.max}, it follows that if $G\in\max\tpp{k}{p}$ there is a causal rule
\mbox{$\R=(l:p\leftarrow p_1,\dotsc, p_m)$}
and c-graphs\dots.
Furthermore, since $P$ is completely labelled, it follows that $l\neq 1$ and then $G<l<1$.
Since $1$ is the only c-graph whose $height$ is $0$
the lemma statement holds vacuous.

In case that $h>0$, we proceed by induction assuming as hypothesis that the lemma statement holds for any $h'<h$.
From Lemma~\ref{lem:tp.causes.composition.max},
there is a causal rule
\mbox{$l:p\leftarrow p_1,\dotsc, p_m$},
and there are
causal graphs $G_{p_1},\dotsc,G_{p_m}$ s.t. each $G_{p_i}\in\max\tpp{k-1}{p_i}$,
$G=G_{\R}\cdot l$
and
$l\not\in V(G_{p_i})$ for any $G_{p_i}$
where $G_{\R}=G_{p_1}*\dotsc*G_{p_m}$.

If $m=0$ then $G=1\cdot l=l$, $height(l)=1$ and
$l\in\max\tpp{k}{p}$ for any $k\geq 1$.
Otherwise, since any path in $G_{p_i}$ is also a path $G_p$,
it is clear that $height(G_{p_i})=h'_{p_i}\leq h$ for any $G_{p_i}$.
Suppose that $h'_{p_i}=h$ for some $G_{p_i}$.
Then there is a simple path $l_1,\dotsc,l_h$ of length $h$ in $G_{p_i}$ and, since $G=G_{\R}\cdot l$, there is an edge $(l_h,l)\in E(G)$. That is $l_1,\dotsc,l_h,l$ is a walk of length $h+1$ in $G$ and, since $l\not\in V(G_{p_i})$, it follows that $l_i\neq l$ with $1\leq i\leq h$. So that  $l_1,\dotsc,l_h,l$ is a simple path of length $h+1$ which contradicts the assumption that $height(G)=h$.
Thus $height(G_{p_i})=h'_{p_i}<h$ for any $G_{p_i}$ and then, by induction hypothesis, $G_{p_i}\in\max\tpp{h'_{p_i}}{p_i}$.

Let $h'=\max\setm{h'_{p_i}}{1\leq i\leq m}<h$.
Since the $T_P$ operator is monotonic and $h'_{p_i}\leq h'$ for any $p_i$, it follows that $\tpp{h_{p_i}'}{p_i}\leq\tpp{h'}{p_i}$ and then
there are casual graphs
$G_{p_1}',\dotsc,G_{p_m}'$ such that each $G_{p_i}'\in\max\tpp{h'}{p_i}$, $G_{p_i}\leq G_{p_i}'$
and
$G'=G_{\R}\cdot l$
where
\mbox{$G_{\R}'=G_{p_1}'*\dotsc*G_{p_m}'$}.
By product and application monotonicity,
it follows that $G\leq G'$,
and, from Lemma~\ref{lem:tp.causes.composition},
it follows that
$G'\in\tpp{h'+1}{p}$.
Since $h'+1\leq h$ it follows that $G'\in\tpp{h}{p}$ and
since $G\leq G'$ it follows that $G\in\tpp{h}{p}$.

Suppose that $G\not\in\max\tpp{h}{p}$.
Then there is $G''\in\max\tpp{h}{p}$ s.t. $G<G''$ and then, since
$h\leq k$, it follows that $G''\in\tpp{k}{p}$ which, since $G<G''$, contradicts the assumption that $G\in\max\tpp{k}{p}$.
Thus, if $G\in\max\tpp{k}{p}$ and $height(G)=h\leq k$ it follows that
$G\in\tpp{h}{p}$.

In case that $k=\omega$, by definition
$\tpp{\omega}{p}=\sum_{i<\omega}\tpp{i}{p}$.
Thus, if $G\in\max\tpp{\omega}{p}$ and $height(G)=h$
then there is some $i<\omega$ s.t.
$G\in\max\tpp{i}{p}$ and $h\leq i$,
and as we already show,
then $G\in\tpp{h}{p}$.
\qed
\end{proof}

\begin{lemma}\label{lem:tp.label.replacing}
Let $P,Q$ two positive causal logic programs such that $Q$ is the result of replacing label $l$ in $P$ by some $u$ (a label or $1$)
then $\tPpp{Q}{k}{p}=\tpp{k}{p}[l\mapsto u]$ for any atom~$p$ and $k\in\set{1,\dotsc,\omega}$.
\end{lemma}
\begin{proof}
In case that $n=0$, $\tPpp{Q}{k}{p}=0$ and
$\tpp{k}{p}=0$ and $0=0[l\mapsto u]$.
That is $\tPpp{Q}{k}{p}=\tpp{k}{p}[l\mapsto u]$.

We proceed by induction on $k$ assuming that
$\tPpp{Q}{k-1}{p}=\tpp{k-1}{p}[l\mapsto u]$ for any atom~$p$ and we will show that
$\tPpp{Q}{}{p}=\tpp{k}{p}[l\mapsto u]$.

Pick $G\in\tpp{k}{p}$ then,
from Lemma~\ref{lem:tp.causes.composition},
there is a rule $l': p\leftarrow q_1,\dotsc, q_m$
and causal graphs
$G_{q_1}$, $\dotsc$, $G_{q_m}$
each one respectively in $\tpp{k-1}{q_i}$ s.t. $G\leq G_R=(G_{q_1}*\dotsc*G_{q_m})\cdot l'$.
Thus, by induction hypothesis,
for every atom $q_i$ and c-graph $G_{q_i}\in\tpp{k-1}{q}$ it holds that
\mbox{$G_{q_i}[l\mapsto u]\in\tPpp{Q}{k-1}{q_i}$}.

Let
$G_\R[l\mapsto u]$ be a c-graph defined as
$G_\R[l\mapsto u]=\big( G_{q_1}[l\mapsto u]*\dotsc*G_{q_m}[l\mapsto u]\big)\cdot l'[l\mapsto u]$.
Then,
since $G\leq G_\R$,
it follows that
$G[l\mapsto u]\leq G_\R[l\mapsto u]$
and then, again from Lemma~\ref{lem:tp.causes.composition}, it follows that 
$G[l\mapsto u]\in\tPpp{Q}{n}{p}$.
That is
$\tpp{k}{p}[l\mapsto u]\subseteq \tPpp{Q}{k}{p}$.

Pick $G\in\tPpp{Q}{k}{p}$ then,
from Lemma~\ref{lem:tp.causes.composition},
there is a rule there is a rule
$l': p\leftarrow q_1,\dotsc, q_m$ and c-graphs
$G_{q_1}$, $\dotsc$, $G_{q_m}$ respectively in $\tpp{k-1}{q_i}$ s.t.
$G\leq G_{\R}$ where
$G_{\R}=(G_{q_1}*\dotsc*G_{q_m})\cdot l'$.

By induction hypothesis, for every atom $q_i$ and graph $G_{q_i}$ it holds that if \mbox{$G_{q_i}\in\tPpp{Q}{k-1}{q_i}$}
then
\mbox{$G_{q_i}\in\tpp{k-1}{q_i}[l\mapsto u]$}.
Thus, it follows that there is a graph $G_{q_i}'\in\tpp{k-1}{q_i}$ such that
\mbox{$G_{q_i}'[l\mapsto u]=G_{q_i}$}.
Let $G_\R'$ be a graph s.t.
$G_\R'=\big( G_{q_1}'*\dotsc*G_{q_m}'\big)\cdot l'$.
From Lemma~\ref{lem:tp.causes.composition}
for every causal graph $G'\leq G_R'$ it holds that $G'\in\tpp{k}{p}$.
Since $G'_{\R}[l\mapsto u]=G_{\R}$ and
$G\leq G_\R$
it follows that $G\leq G_{\R}[l\mapsto u]$ and, since $G_{\R}\in\tpp{k}{p}$,
it follows that 
$G\in\tpp{k}{p}[l\mapsto u]$.
Consequently
\mbox{$\tpp{k}{p}[l\mapsto u]\supseteq \tPpp{Q}{n}{p}$} and then
\mbox{$\tpp{k}{p}[l\mapsto u]=\tPpp{Q}{n}{p}$}.

In case that $k=\omega$, by definition
$\tpp{\omega}{p}[l\mapsto u]
	=\sum_{i<\omega}\tpp{}{i}{p}[l\mapsto u]$
and as we alerady show
$\tpp{}{i}{p}[l\mapsto u]=\tPpp{Q}{i}{p}$
for all integer $i<\omega$, so that,
their sum is also equal
and consequently
$\tpp{\omega}{p}[l\mapsto u]=\tPpp{Q}{\omega}{p}$.
\end{proof}

\begin{proofof}{Theorem~\ref{theorem:tp.properties}}
Let $P'$ be a positive, completely labelled causal program with the same rules as $P$.
From Lemma~\ref{lem:tp.properties} it follows that
($i$) $\mathit{lfp}(T_{P})$ and $\mathit{lfp}(T_{P'})$ are respectively the least model of the programs $P$ and $P'$, and
($ii$) $\mathit{lfp}(T_P)=\tp{\omega}$ and $\mathit{lfp}(T_{P'})=\tPp{P'}{\omega}$.

Futhermore, it is clear that if $P$ is an infinite program, i.e.
$n=\omega$, then $\tPp{P'}{n}=\tPp{P'}{\omega}$.
Otherwise, by definition it holds that
$\tPp{P'}{n}\leq\tPp{P'}{\omega}$.
Suppose
\mbox{$\tPp{P'}{n}<\tPp{P'}{\omega}$}.
Then there is some atom $p$ and c-graph $G\in\tPpp{P'}{\omega}{p}$ such that $G\not\in\tPpp{P'}{n}{p}$.
The longest simple path in $G$ must be smaller than the number of its vertices and this must be smaller than the number of labels of the program which in its turn is equal to the number of rules $n$, i.e.
\mbox{$height(G)=h\leq n$}.
From Lemma~\ref{lem:tp.height} it follows that $G\in\tPpp{P'}{h}{p}$ and since $h\leq n$ it follows that
\mbox{$\tPpp{P'}{h}{p}\subseteq\tPpp{P'}{n}{p}$}
and so that $G\in\tPpp{P'}{n}{p}$
which is a contradiction with the assumption that $G\in\tpp{\omega}{p}$ but $G\not\in\tpp{n}{p}$.
Thus $\tPp{P'}{n}=\tPp{P'}{\omega}$.

Furthermore, from Lemma~\ref{lem:tp.label.replacing},
$\tpp{k}{p}
	=\tPpp{P'}{k}{p}[l'_1\mapsto l_1]\dotsc[l'_n\mapsto l_n]$
for $k\in\set{n,\omega}$ and
where $l_1',\dotsc,l_n'$ are the labels of rules of $P'$ and $l_1,\dotsc,l_n$ are the correspondent labels of such rules in $P$.
Thus, since $\tPp{P'}{n}=\tPp{P'}{\omega}$, it follows that
$\tp{n}=\tp{\omega}$.
\end{proofof}

\begin{lemma}\label{lem:proof.graph.from.subproofs}
For any proof $\pi(p)$ it holds that
\begin{align*}
graph\Big( \cfrac{\pi(q_1),\dotsc,\pi(q_m)}{p} \ (l) \Big)
=\big( graph(\pi(q_1))*\dotsc*graph(\pi(q_1)) \big) \cdot l
\end{align*}
\end{lemma}
\begin{proof}
We proceed by structural induction assuming that for every proof in the antecedent $\pi(q_i)$ and every label $l'\in V(graph(\pi(q_i)))$ there is an edge $(l',label{(\pi(q_i))})\in E(graph(\pi(q_i)))$.

By definition $graph(\pi(p))=G_{\pi(p)}^*$ is the reflexive and transitive closure of $G_{\pi(p)}$ and then
\begin{align*}
graph(\pi(p))&=\Big(
	\bigcup\setbm{graph(\pi(q_i))}{1 \leq i\leq m}
	\cup\setbm{(label(\pi(q_i),l)}{ 1 \leq i\leq m}
	\Big)^*
\end{align*}
Thus, $graph(\pi(p))\geq \prod\setbm{graph(\pi(q_i))}{1\leq i\leq m}\cdot l$ and remain to show that
for every atom $q_i$ and label $l'\in V(graph(\pi(q_i)))$
the edge $(l',l)\in E(graph(\pi(p)))$.
Indeed, since by induction hypothesis there is an edge
$(l',label(\pi(q_i)))\in E(graph(\pi(q_i)))\subseteq E(graph(\pi(p))$,
the fact that the edge $(label(\pi(q_i),l)\in E(graph(\pi(p)))$ and since $graph(\pi(p))$ is closed transitively,
it follows that $(l',l)\in E(graph(\pi(p)))$.\qed
\end{proof}

\begin{lemma}\label{lem:tp.proof->graph}
Let $P$ be a positive, completely labelled program and $\pi(p)$ be a proof for $p$ w.r.t. $P$.
Then it holds that
$graph(\pi_p)\in\tpp{h}{p}$
where $h$ is the height of $\pi(p)$ which is recursively defined as
\begin{align*}
height(\pi)=1+
	\max\setm{height(\pi')}{\pi' \text{ is a sub-proof of }\pi}
\end{align*}
\end{lemma}
\begin{proof}
In case that $h=1$ the antecedent of $\pi(p)$ is empty, i.e.
\begin{align*}
\pi(p)= \cfrac{\top}{p} \ (l)
\end{align*}
where $l$ is the label of the fact $(l:p)$.
Then $graph(\pi(p))=l$.
Furthermore, since the fact $(l:p)$ is in the program $P$,
it follows that $l\in \tpp{1}{p}$.

In the remain cases, we proceed by structural induction assuming that
for every natural number $h\leq n-1$, atom $p$
and proof $\pi(p)$ of $p$ w.r.t. $P$ whose $height(\pi(p))=h$
it holds that $graph(\pi(p))\in\tpp{h}{p}$
and we will show it in case that $h=n$.

Since $height(\pi_p)>1$ it has a non empty antecedent, i.e.
\begin{align*}
\pi(p)= \cfrac{\pi(q_1),\dotsc,\pi(q_m)}{p} \ (l)
\end{align*}
where $l$ is the label of the rule $l:p\leftarrow q_1,\dotsc,q_m$.
By $height$ definition, for each $q_i$ it holds that $height(\pi(q_i))\leq n-1$ and so that, by induction hypothesis, $graph(\pi(q_i))\in\tpp{h-1}{q_i}$.
Thus, from Lemmas~\ref{lem:tp.causes.composition}~and~\ref{lem:proof.graph.from.subproofs}, it follows respectively that
\begin{IEEEeqnarray*}{rCcCl}
&&\prod\setm{graph(\pi(q_i))}{1\leq i\leq m}\cdot l &\in&\tpp{h}{p}
\\
graph(\pi(p)) &=& \prod\setm{graph(\pi(q_i))}{1\leq i\leq m}\cdot l
\end{IEEEeqnarray*}
That is, $graph(\pi(p))\in\tpp{h}{p}$. \qed
\end{proof}

\begin{lemma}\label{lem:tp.proof<-graph}
Let $P$ be a positive, completely labelled program and $\pi(p)$ be a proof of $p$ w.r.t. $P$.
For every atom $p$ and maximal causal graph $G\in\tpp{\omega}{p}$
there is a non-redundant proof $\pi(p)$ for $p$ w.r.t. $P$ s.t.
$graph(\pi(p))=G$.
\end{lemma}
\begin{proof}
From Lemma~\ref{lem:tp.causes.composition.max}
for any maximal graph $G\in\tpp{k}{p}$,
there is a rule
$l:p\leftarrow q_1,\dotsc, q_m$ and maximal graphs $G_{q_1}\in \tpp{h-1}{q_1},\dotsc, G_{q_m}\in \tpp{k-1}{q_m}$ s.t.
\begin{align*}
G=(G_{q_1}*\dotsc*G_{q_m} )\cdot l
\end{align*}
Furthermore, we assume as induction hypothesis that
for every atom $q_i$ there is a non redundant proof $\pi(q_i)$ for $q_i$ w.r.t. $P$ s.t.
$graph(\pi(q_i))=G_{q_i}$. Then $\pi(p)$  defined as
\begin{align*}
\pi(p)= \cfrac{\pi(q_1),\dotsc,\pi(q_m)}{p} \ (l)
\end{align*}
is a proof for $p$ w.r.t. $P$ which holds
$graph(\pi(p))=G_{p}$ (from Lemma~\ref{lem:proof.graph.from.subproofs})
and $height(\pi(p)\leq h$.
Furthermore, suppose that $\pi(p)$ is redundant, i.e.
there is a proog $\pi'$ for $p$ w.r.t $P$ such that
$graph(\pi(p))<graph(\pi')$.
Let $h=height(\pi')$.
Then, from Lemma~\ref{lem:tp.proof->graph}, it follows that
$graph(\pi')\in\tpp{h}{p}$ and then $graph(\pi')\in\tpp{\omega}{p}$
which contradicts the hypothesis that $G$ is maximal in $\tpp{\omega}{p}$.\qed
\end{proof}

\begin{proofof}{Theorem~\ref{th:proofs}}
From Theorem~\ref{theorem:tp.properties} it follows that the least model $I$ is equal to $\tp{\omega}$.
For the only if direction,
from Lemma~\ref{lem:tp.proof<-graph},
it follows that
for every maximal c-graph $G\in I(p)=\tpp{\omega}{p}$
there is a non-redundant proof $\pi(p)$ for $p$ w.r.t $P$ s.t.
$G=graph(\pi(p))$.
That is, $\pi(p)\in \Pi_p$ and then $G=graph(\pi(p))\in graph(\Pi_p)$.
For the if direction,
from Lemma~\ref{lem:tp.proof->graph},
for every $G\in graph(\Pi_p)$,
i.e. $G=graph(\pi(p))$ for some non-redundant proof $\pi(p)$ for $p$ w.r.t. $P$,
it holds that $G\in\tpp{\omega}{p}$ and so that $G\in I(p)$.
Furthermore, suppose that $G$ is not maximal, i.e. there is a maximal
c-graph $G'\in I(p)$ s.t. $G<G'$ and a proof $\pi'$ for $p$ w.r.t. $P$ s.t. $graph(\pi')=G'$ which contradicts that $\pi(p)$ is non-redundant.\qed
\end{proofof}

\begin{lemma}\label{lem:term.label.replacing}
Let $t$ be a causal term.
Then $value(t[l\mapsto u])= value(t)[l\mapsto u]$.
\end{lemma}
\begin{proof}
We proceed by structural induction.
In case that $t$ is a label. If $t=l$
then $value(l[l\mapsto u]=value(u)=\downarrow u=value(l)[l\mapsto u]$.
If $t=l'\neq l$ then
$value(l'[l\mapsto u]=value(l')=\downarrow l'=value(l')[l\mapsto u]$.
In case that $t=\prod T$  it follows that
$value(\prod T[l\mapsto u])=\bigcap\setm{value(t'[l\mapsto u])}{t'\in T})$ and by induction hypothesis $value(t'[l\mapsto u])=value(t')[l\mapsto u]$.
Then
$value(\prod T[l\mapsto u])
	=\bigcap\setm{value(t')[l\mapsto u]}{t'\in T})
	=value(\prod T)[l\mapsto u]$.
The cases for $t=\sum T$  is analogous.
In case that $t=t_1\cdot t_2$ it follows that
$value(t[l\mapsto u])
	=value(t_1[l\mapsto u])\cdot value(t_2[l\mapsto u])
	=value(t_1)[l\mapsto u]\cdot value(t_2)[l\mapsto u]
	=value(t)[l\mapsto u]$
\end{proof}

\begin{proofof}{Theorem~\ref{theorem:least.model.label.replacing}}
From Theorem~\ref{theorem:tp.properties}, models $I$ and $I'$ are respectively equal to $\tp{\omega}$ and $\tPp{P'}{\omega}$.
Furthermore, from Lemma~\ref{lem:tp.label.replacing}, it follows that
$\tPpp{P'}{\omega}{p}=\tpp{\omega}{p}[l\mapsto u]$ for any atom $p$.
Lemma~\ref{lem:term.label.replacing} shows that the replacing can be done in any causal term without operate it.
\end{proofof}

\begin{proofof}
{Theorem~\ref{theorem:least.model.classical.correspondence}}
It is clear that if every rule in $P$ is unlabelled, i.e. $P=P'$, then their least model assigns $0$ to every $false$ atom and $1$ to every $true$ atom, so that their least models coincide with the classical one, i.e. $I=I'$ and then $I^{cl}=I=I'$.
Otherwise, let $P_n$ be a program where $n$ rules are labelled.
We can build a program $P_{n-1}$ removing one label $l$ and, from Theorem~\ref{theorem:least.model.label.replacing},
it follows that $I_{n-1}=I_n[l\rightarrow 1]$.
By induction hypothesis the corresponding classical interpretation of least model of $P_{n-1}$ coincides with the least model of the unlabelled program, i.e. $I_{n-1}^{cl}=I'$,
and then $I_n[l\mapsto 1]^{cl}=I_{n-1}^{cl}=I'$.
Furthermore,
for every atom $p$ and c-graph $G$ it holds that $G\in I_n(p)$ iff
$G[l\mapsto 1]\in I_n[l\mapsto 1](p)$.
Simple remain to note that $value(z)=\emptyset$,
so that $I_n(p)=0$ iff $I_n[l\mapsto 1](p)=0$ and consequently $I_n^{cl}=I_{n}[l\mapsto 1]^{cl}=I'$.\qed
\end{proofof}

\begin{proofof}{Theorem~\ref{th:csm}}
By definition $I$ and $I^{cl}$ assigns $0$ to the same atoms, so that
$P^I=P^{I^{cl}}$.
Furthermore let $Q$ (instead of $P'$ for clarity) be the unlabelled version of $P$. Then $Q^{I^{cl}}$ is the unlabelled version of $P^I$.
$(1)$ Let $I$ be a stable model of $P$ and $J$ be the least model of $Q^{I^{cl}}$.
Then, $I$ is the least model of $P^I$
and, from Theorem~\ref{theorem:least.model.classical.correspondence},
it follows that $I^{cl}=J$, i.e. $I^{cl}$ is a stable model of $Q$.
$(2)$ Let $I'$ is a stable model of $Q$ and $I$ be the least model of $P^{I'}$.
Since $I'$ is a stable model of $Q$, by definition it is the least model of $Q^{I'}$, furthermore, since $Q^{I'}$ is the unlabelled version of $P^{I'}$ it follows, from Theorem~\ref{theorem:least.model.classical.correspondence}, that $I^{cl}=I'$.
Note that $P^I=P^{I^{cl}}=P^{I'}$. Thus $I$ is a stable model of $P$.\qed
\end{proofof}

\newpage
\section{Algebraic completeness}

We will show that causal terms with the algebraic properties reflected in Figures~\ref{fig:DBLattice.ext}~and~\ref{fig:appl.ext} are correct and complete with respect to the algebra of causal values.

\begin{figure}[htbp]
\begin{center}
$
\begin{array}{c}
\hbox{\em Associativity} \\
\hline
\begin{array}{r@{\ }c@{\ }r@{}c@{}l c r@{}c@{}l@{\ }c@{\ }l@{\ }}
t & + & (u & + & w) & = & (t & + & u) & + & w\\
t & * & (u & * & w) & = & (t & * & u) & * & w
\end{array}
\end{array}
$
\ \
$
\begin{array}{c}
\ \ \ \ \hbox{\em Commutativity}\ \ \ \ \\
\hline
\begin{array}{r@{\ }c@{\ }l c r@{\ }c@{\ }l@{\ }}
t & + & u & = & u & + & t\\ 
t & * & u & = & u & * & t
\end{array}
\end{array}
$
\ \
$
\begin{array}{c}
\hbox{\em Absorption} \\
\hline
\begin{array}{c c r@{\ }c@{\ }r@{}c@{}l@{\ }}
t & = & t & + & (t & * & u)\\
t & = & t & * & (t & + & u)
\end{array}
\end{array}
$
\ \
\\
\vspace{10pt}
$
\begin{array}{c}
\hbox{\em Distributive} \\
\hline
\begin{array}{r@{\ }c@{\ }r@{}c@{}l c r@{}c@{}l@{\ }c@{\ }r@{}c@{}l@{}}
t & + & (u & * & w) & = & (t & + & u) & * & (t & + & w)\\
t & * & (u & + & w) & = & (t & * & u) & + & (t & * & w)
\end{array}
\end{array}
$
\ \
$
\begin{array}{c}
Identity \\
\hline
\begin{array}{rcr@{\ }c@{\ }l@{\ }}
t & = & t & + & 0\\
t & = & t & * & 1
\end{array}
\end{array}
$
\ \
$
\begin{array}{c}
\hbox{\em Idempotence} \\
\hline
\begin{array}{rcr@{\ }c@{\ }l@{\ }}
t & = & t & + & t\\
t & = & t & * & t
\end{array}
\end{array}
$
\ \ 
$
\begin{array}{c}
\hbox{\em Annihilator} \\
\hline
\begin{array}{rcr@{\ }c@{\ }l@{\ }}
1 & = & 1 & + & t\\
0 & = & 0 & * & t
\end{array}
\end{array}
$
\end{center}
\caption{Sum and product satisfy the properties of a completely distributive lattice.}
\label{fig:DBLattice.ext}
\end{figure}

\begin{figure}[htbp]
\begin{center}
\newcommand{\titleSep}{0pt}
\newcommand{\contentSep}{-10pt}
\newcommand{\rowSep}{5pt}
$
\begin{array}{c}
\hbox{\em Absorption}\vspace{\titleSep}\\
\hline\vspace{\contentSep}\\
\begin{array}{r@{\ }c@{\ }c@{\ }c@{\ }l c r@{\ }c@{\ }r@{\ }c@{\ }c@{\ }c@{\ }c@{\ }l@{\ }}
&& t &&& = & t & + & u & \cdot & t & \cdot & w \\
u & \cdot & t & \cdot & w & = & t & * & u & \cdot & t & \cdot & w
\end{array}
\end{array}
$
\ \ \ \
$
\begin{array}{c}
\hbox{\em Associativity}\vspace{\titleSep}\\
\hline\vspace{\contentSep}\\
\begin{array}{r@{\ }c@{\ }r@{}c@{}l c r@{}c@{}l@{\ }c@{\ }l@{\ }}
t & \cdot & (u & \cdot & w) & = & (t & \cdot & u) & \cdot & w\\
\\
\end{array}
\end{array}
$
\ \ \ \
$
\begin{array}{c}
\hbox{\em Identity}\vspace{\titleSep}\\
\hline\vspace{\contentSep}\\
\begin{array}{rc r@{\ }c@{\ }l@{\ }}
t & = & 1 & \cdot & t\\
t & = & t & \cdot & 1
\end{array}
\end{array}
$
\ \ \ \
$
\begin{array}{c}
\hbox{\em Annihilator}\vspace{\titleSep}\\
\hline\vspace{\contentSep}\\
\begin{array}{rc r@{\ }c@{\ }l@{\ }}
0 & = & t & \cdot & 0\\
0 & = & 0 & \cdot & t\\
\end{array}
\end{array}
$
\\
\vspace{\rowSep}
$
\begin{array}{c}
\hbox{\em Indempotence}\vspace{\titleSep}\\
\hline\vspace{\contentSep}\\
\begin{array}{r@{\ }c@{\ }l@{\ }c@{\ }l }
l & \cdot & l  & = & l\\
\\
\\
\end{array}
\end{array}
$
\hspace{.1cm}
$
\begin{array}{c}
\hbox{\em Addition\ distributivity}\vspace{\titleSep}\\
\hline\vspace{\contentSep}\\
\begin{array}{r@{\ }c@{\ }r@{}c@{}l c r@{}c@{}l@{\ }c@{\ }r@{}c@{}l@{}}
t & \cdot & (u & + & w) & = & (t & \cdot & u) & + & (t & \cdot & w)\\
( t & + & u ) & \cdot & w & = & (t & \cdot & w) & + & (u & \cdot & w)\\ \\
\end{array}
\end{array}
$
\hspace{.1cm}
$
\begin{array}{c}
\hbox{\em Product\ distributivity}\vspace{\titleSep}\\
\hline\vspace{\contentSep}\\
\begin{array}{rcl}
c \cdot d \cdot e & = & (c \cdot d) * (d \cdot e) \ \ \ \hbox{with} \ d \neq 1 \\
c \cdot (d*e)     & = & (c \cdot d) * (c \cdot e) \\
(c*d) \cdot e     & = & (c \cdot e) * (d \cdot e)
\end{array}
\end{array}
$
\end{center}
\vspace{-5pt}
\caption{Properties of the `$\cdot$' operator ($c,d,e$ are terms without `$+$' and $l$ is a label).}
\label{fig:appl.ext}
\end{figure}

\noindent
First note that the algebraic properties for ``$*$'' and ``$+$'' reflected Figure~\ref{fig:DBLattice.ext} are the common algebraic properties satisfy by distributive lattices. Then their correctness follows directly from Theorem~\ref{theorem:freelattice}.

\begin{proposition}[Application homomorphism]
\label{prop:down.application.homomorphism}
The mapping $\downarrow:\causes,\longrightarrow\VLb$ is an homomorphism between $\tuple{\causes,\cdot}$ and $\tuple{\VLb,\cdot}$,
\end{proposition}
\begin{proof}
Take any c-graphs $G_1$ and $G_2$, then
\begin{align*}
\down{G_1}\ \cdot\down{G_2}
	& \ = \ \down{\setm{(G_1'\cdot G_2')}
		{ G_1' \in \ \down{G_1} \text{ and } G_2'\in \ \down{G_2} } }
\\
	& \ = \ \down{\setm{(G_1'\cdot G_2')}
		{ G_1' \leq G_1 \text{ and } G_2' \leq G_2 } }
	  \ = \ \down{(G_1\cdot G_2)}
\end{align*}
Note that, from Proposition~\ref{prop:graphs.monotonicity},it follows that $G_1'\cdot G_2'\leq G_1\cdot G_2$. That is, $\downarrow$ in an homomorphism between $\tuple{\causes,\cdot}$ and $\tuple{\VLb,\cdot}$.
\end{proof}

\begin{theorem}[Homomorphism between causal graphs and values]\label{thm:homomorphism.graphs->values}
The mapping $\downarrow:\causes,\longrightarrow\VLb$ is an injective homomorphism between $\tuple{\causes,+,*\cdot}$ and $\tuple{\VLb,+,*,\cdot}$,
\end{theorem}
\begin{proof}
This follows directly from the Theorem~\ref{theorem:freelattice} and Proposition~\ref{prop:down.application.homomorphism}
\end{proof}

\begin{proposition}\label{prop:term.nosum.principal}
For every causal term $t$ without addition `$+$' there exists a causal graph $G$ such that $t = \ \down{G}$.
\end{proposition}
\begin{proof}
In case that $t$ is a label, by definition is the principal ideal $\down{t}$. We proceed by structural induction. In case that $t=\prod S$, by induction hypothesis, for every $t_i\in S$ there is some causal graph $G_i$ such that $t_i=\ \down{G_i}$ and from Theorem~\ref{theorem:freelattice} it follows that
$$
t=\prod S=\prod\setm{\down{G_i}}{t_i\in S}=\ \down{\setm{ G_i }{ t_i\in S } }$$
which is a principal ideal. In case that $t=t_1\cdot t_2$, by induction hypothesis, $t_1=\ \down{G_1}$ and $t_2=\ \down{G_2}$ and from Proposition~\ref{prop:down.application.homomorphism}, it follows that
$$
t=t_1\cdot t_2=\ \down{G_1}\ \cdot\down{G_2}=\ \down{(G_1\cdot G_2)}$$
which is a principal idea. 
\end{proof}

\begin{lemma}
Let $G_u$ and $G_v$ be two causal graphs. If
$G_u\supseteq G_v$, then \mbox{$term(G_u)\leq term(G_v)$} follows from the equivalence laws reflected in Figures~\ref{fig:DBLattice.ext}~and~\ref{fig:appl.ext}.
\end{lemma}

\begin{proof}
By definition $term(G_u) * term(G_v)$ is equal to
\begin{align}
	\prod\setbm{ u_1 \cdotl u_2 }{ (u_1,u_2) \in G_u }
		* \prod\setbm{ v_1 \cdotl v_2 }{ (v_1,v_2) \in G_v }
		\label{eq:1:lem:graph.supseteq->term.leq}
\end{align}
and since, $G_u\supseteq G_v$, it follows that every edge $(v_1,v_2) \in G_v$ also belong to $G_u$. Thus applying associative and conmutative laws we can rewrite \eqref{eq:1:lem:graph.supseteq->term.leq} as
\begin{align*}
\prod\setbm{ u_1 \cdotl u_2 }{ (u_1,u_2) \in G_u \backslash G_v }
* \prod\setbm{ (v_1 \cdotl v_2) * (v_1 \cdotl v_2) }{ (v_1,v_2) \in G_v }
\end{align*}
that, by product idempotence, is clearly equivalent to
\begin{align*}
\prod\setbm{ u_1 \cdotl u_2 }{ (u_1,u_2) \in G_u \backslash G_v }
* \prod\setbm{ (v_1 \cdotl v_2)  }{ (v_1,v_2) \in G_v }
\end{align*}
That is, $term(G_u) * term(G_v) = term(G_u)$ which implies $term(G_u) \leq term(G_v)$.
\end{proof}

\begin{lemma}\label{lem:value.subseteq->term.leq}
Let $U$ and $V$ be two causal values. If $U \subseteq V$, then $term(U)\leq term(V)$ follows from the equivalence laws reflected in Figures~\ref{fig:DBLattice.ext}~and~\ref{fig:appl.ext}.
\end{lemma}

\begin{proof}
By definition $term(U)+term(V)$ is equal to
\begin{align}
\sum\setm{ term(G_u) }{ G_u \in U }
+ \sum\setm{ term(G_v) }{ G_v \in V }
   \label{eq:1:lem:value.subseteq->term.leq}
\end{align}
and since, $U\subseteq V$, it follows that every c-graph $G_u \in U$ also belong to $V$. Thus applying associative and conmutative laws we can rewrite \eqref{eq:1:lem:value.subseteq->term.leq} as
\begin{align*}
	\sum\setbm{ G_u }{ (u_1,u_2) \in U }
		+ \sum\setbm{ G_v + G_v }{ G_v \in V \backslash U }
\end{align*}
that, by summ idempotence, is clearly equivalent to
\begin{align*}
	\sum\setbm{ G_u }{ (u_1,u_2) \in U }
		+ \sum\setbm{ G_v }{ G_v \in V \backslash U }
\end{align*}
That is, $term(U) + term(V) = term(V)$ which implies $term(U) \leq term(V)$.
\end{proof}

\begin{proposition}[Application associativity]\label{prop:values.appl.associative}
Let $T$, $U$ and $W$ be three causal values. Then, it holds that $U \cdot (T \cdot W)  \ = \ ( U \cdot T ) \cdot W$ and further
\mbox{$U \cdot T \cdot W \ = \ \down{\setbm{ G_U \cdot G_T \cdot G_W}{G_U \in U,\ G_T \in T \text{ and } G_W \in W}}$}.
\end{proposition}
\begin{proof}
By definition it follows that
\begin{IEEEeqnarray*}{*x+r C l}
&(U\cdot T)\cdot W
	& = &
	\down{\setm{ G_U \cdot G_T }{ G_U \in U \text{ and } G_T \in T}}
		\cdot W
	\\
&	& = &
	\down{\setm{ G' \cdot G_W }
		{ G_U \in U,\ G_T \in T,\ G'\leq G_U\cdot G_T \text{ and } G_W \in W}}
	\\
&	& = &
	\down{\setm{ (G_U\cdot G_T) \cdot G_W }
		{ G_U \in U,\ G_T \in T \text{ and } G_W \in W}}
		\label{eq:prop:values.appl.associative.1}
\\
\IEEEeqnarraymulticol{4}{l}{
	\text{In the same way, it also follows that}
}
\\
& U \cdot (T \cdot W)
	& = &
	\down{\setm{ G_U\cdot (G_T \cdot G_W) }
		{ G_U \in U,\ G_T \in T \text{ and } G_W \in W}}
	\label{eq:prop:values.appl.associative.2}
\end{IEEEeqnarray*}
Then it is enough to show that
$(G_U\cdot G_T) \cdot G_W \ = \ G_U\cdot (G_T \cdot G_w) = G_U \cdot G_T \cdot G_W$ which holds due to Proposition~\ref{prop:graph.appl.associative}.
\end{proof}

\begin{proposition}[Application absorption]
Let $T$, $U$ and $W$ be three causal values. Then $T = T + \ U \cdot T \cdot W$ and $U \cdot T \cdot W = T * U \cdot T \cdot W$
\end{proposition}
\begin{proof}
From Proposition~\ref{prop:values.appl.associative} it follows that
\begin{align*}
U \cdot T \cdot W
	&= \ \down{\setbm{ G_U \cdot G_T \cdot G_W}{G_U \in U,\ G_T \in T \text{ and } G_W \in W}}
\end{align*}
Furthermore, for every c-graph $G_T\in T$, it holds that $G_U\cdot G_T \cdot T_W \leq G_T$. Then, since $T$ is an ideal, it follows that $G_U\cdot G_T \cdot T_W \in T$ and consequently $U\cdot T \cdot W \subseteq T$. Thus \mbox{$U \cdot T \cdot W \cup T = T$} and \mbox{$U \cdot T \cdot W \cap T = U \cdot T \cdot W$} and, by definition, these equalities can be rewritten as \mbox{$U \cdot T \cdot W + T = T$} and \mbox{$U \cdot T \cdot W * T = U \cdot T \cdot W$}.
\end{proof}

\begin{proposition}[Application distributivity w.r.t. additions]
Let $T$, $U$ and $W$ be three causal values. Then, it holds that $U \cdot (T + W)  \ = \ ( U \cdot T ) + (U \cdot W$ and 
\mbox{$( U + T ) \cdot W \ = \ ( U \cdot W ) + ( T \cdot W)$}.
\end{proposition}
\begin{proof}
By definition, it follows that
\begin{align*}
(U \cdot T) + (U\cdot W)
	&= (U \cdot T) \cup (U\cdot W)
\\
	&= \ \down{ \setbm{ G_U \cdot G_T }{ G_U \in U \text{ and } G_T\in T} }
	\ \cup \ \down{ \setbm{ G_U \cdot G_T }{ G_U \in U \text{ and } G_W\in W} }
\\
	&= \ \down{ \setbm{ G_U \cdot G' }{ G_U \in U \text{ and } G'\in T\cup W} }
	= U \cdot ( T \cup W ) = U \cdot ( T + W )
\end{align*}
Furthermore \mbox{$( U + T ) \cdot W \ = \ ( U \cdot W ) + ( T \cdot W)$} holds symmetrically.
\end{proof}

\noindent
Product distributivity laws follows from Propositions~\ref{prop:graphs.appl.distr.over.prods}~and~\ref{prop:graphs.appl.distr.over.prods.cont}.

\begin{corollary}
The applications distributivity properties reflected in Figure~\ref{fig:appl.ext} hold.
\end{corollary}
\begin{proof}
From proposition~\ref{prop:term.nosum.principal}, it follows that $c$, $d$ and $e$ are the principal ideals $\down{G_c}$, $\down{G_d}$ and $\down{G_e}$ for some c-graphs $G_c$, $G_d$ and $G_e$. Furthermore, from Proposition~\ref{prop:graphs.appl.distr.over.prods} distributivity holds for causal graphs, i.e. \mbox{$G_c \cdot (G_d * G_e) = (G_c \cdot G_d) * (G_c \cdot G_e)$} and \mbox{$(G_c * G_d) \cdot G_e = (G_c \cdot G_e) * (G_d \cdot G_e)$} and from Theorem~\ref{thm:homomorphism.graphs->values} it follows that this also holds for they principal ideals $\down{G_c}$, $\down{G_d}$ and $\down{G_e}$. In the same way, from Proposition~\ref{prop:graphs.appl.distr.over.prods.cont}, it follows $c \cdot d \cdot e = ( c \cdot d ) * ( d \cdot e )$.
\end{proof}

\begin{proposition}[Application identity and annihilator]
Given a causal value $T$, it holds that $T = 1 \cdot T$, $T= T \cdot 1$, $0 = T \cdot 0$ and $0 = 0 \cdot T$.
\end{proposition}
\begin{proof}
Note that $1$ and $0$ respectively correspond to $\causes$ and $\emptyset$ and by definition it follows that
\begin{align*}
1 \cdot T &= \ \down{ \setbm{G \cdot G_T}{ G\in\causes \text{ and } G_T \in T} } = T
\\
0 \cdot T &= \ \down{ \setbm{G \cdot G_T}{ G\in\emptyset \text{ and } G_T \in T} } = \emptyset = 0
\end{align*}
The other cases are symmetric.
\end{proof}

\begin{proposition}[Application idempotentce]
Given a label $l$, it holds that $l\cdot l = l$
\end{proposition}
\begin{proof}
By definition $l$ corresponds to the causal value $\down{G_l}$ where $G_l$ is the causal graphs which contains the only edge $(l,l)$. Furthermore $G_l \cdot G_l = G_l$ and from Theorem~\ref{thm:homomorphism.graphs->values} it follows that \mbox{$\down{G_l} \ = \ \down{G_l} \ \cdot \down{G_l}$} and consequently $l= l \cdot l$.
\end{proof}

\begin{proposition}[Causal term representation]
For every causal value $T$ is equal to $\sum\setbm{ term(G) }{G \in T}$.
\end{proposition}
\begin{proof}
For every causal term $T$, it holds that $T=\sum\setm{ \down{G} }{G \in T }$. Furthermore, from Proposition~\ref{prop:graph.NF}, it follows that every causal graph $G$ is equal to $\prod\setm{l_1 \cdot l_2}{(l_1,l_2) \text{ is an edge of } G }$ and then, from Theorem~\ref{thm:homomorphism.graphs->values}, it holds that $\down{G} = \prod\setm{l_1 \cdot l_2}{(l_1,l_2) \text{ is an edge of } G }$. Consequently $T=\sum\setm{ term(G) }{ G\in T}$.
\end{proof}

\begin{proposition}\label{prop:term.nosum.normal.form}
Every causal term without sums $c$ can be rewritten as $normal(c)\eqdef\prod\setm{l_1 \cdot l_2)}{ (l_1,l_2) \text{ is an edge of } G }$ for some casual graph $G$ using the algebraic equivalences in Figures~\ref{fig:DBLattice.ext}~and~\ref{fig:appl.ext}.
\end{proposition}
\begin{proof}
We start rewritten $c$ as $\prod C$ where every causal term $x\in C$ is in the form of $l_1\cdot\dotsc\cdot l_n$ by by applying the distributive law until no product is in the scope of ``$\cdot$''. Then we remove each occurrence of $1$ applying the identity law and we replace term $x$ by $x \cdot x$ when $x$ is a label. That is, we have $C'$ s.t. every $x' \in C$ is equal to $x'=l_1\cdot\dotsc\cdot l_n$ with $n>1$ and $l_i\neq 1$.
Then we rewrite each $x'\in C'$ as $l_1\cdot l_2*l_2\cdot l_3*\dotsc*l_{n-1}\cdot l_n$ by successively application of the equivalence $c \cdot d \cdot e= c\cdot d * d\cdot e$. Finally we add every transitive edge applying the equivalence $c\cdot d*d\cdot e=c \cdot d * d\cdot e * c \cdot e$ (Proposition~\ref{prop:term.nosum.transitivity}).
\end{proof}

\begin{proposition}\label{prop:term.normal.form}
Every causal term $t$ can be rewritten as $normal(t)\eqdef\setm{normal(c)}{c\in S}$ for some set of terms without sums $S$ using the algebraic equivalences in Figures~\ref{fig:DBLattice.ext}~and~\ref{fig:appl.ext}.
\end{proposition}
\begin{proof}
We only have to rewrite $t$ as a causal term where sums are not in the scope of ``$*$'' and ``$\cdot$'' by successively application of distributivity over sums. Then the statement holds from Proposition~\ref{prop:term.nosum.normal.form}
\end{proof}

\begin{proposition}
Let $c$ and $d$ be causal terms without ``$+$''. Then $c \leq d$ iff $c = c*d$ follows from the algebraic equivalences in Figures~\ref{fig:DBLattice.ext}~and~\ref{fig:appl.ext}.
\end{proposition}
\begin{proof}
From Proposition~\ref{prop:term.nosum.normal.form}, $c$ and $d$ can be  rewrite as $c'=\prod\setm{l_1 \cdot l_2)}{ (l_1,l_2) \text{ is an edge of } G_c }$ and $d'=\prod\setm{l_1 \cdot l_2)}{ (l_1,l_2) \text{ is an edge of } G_d }$ for some causal graphs $G_c$ and $G_d$. Furthermore, from Proposition~\ref{prop:graph.NF} and Theorem~\ref{thm:homomorphism.graphs->values}, it holds that $\down{G_c}=c'$ and $\down{G_d}=d'$.
Thus, by definition $c\leq d$ iff $\down{G_c}\leq \ \down{G_d}$ iff $\down{G_c}= \down{G_c} \ \cap \down{G_d}$ iff $\down{G_c} = \ \down{G_c} \ * \down{G_d}$ iff $c' = c' * d'$. Let us see that $c' = c'*d'$ can be follows from the algebraic equivalences.
\end{proof}

\label{lastpage}
\end{document}